\documentclass[11pt,a4paper]{article}
\usepackage[T1]{fontenc}
\usepackage[utf8]{inputenc}
\usepackage{lmodern}
\usepackage[margin=1in]{geometry}
\usepackage{microtype}
\usepackage{newpxtext,newpxmath}
\usepackage{amsmath,mathtools,bm}

\usepackage{amsthm}
\usepackage{amsfonts,amssymb}
\usepackage{dsfont}
\usepackage{graphicx}
\graphicspath{{images/}{figures/}}
\usepackage{booktabs,threeparttable,tabularx,array,colortbl,multirow,makecell}
\usepackage{xcolor}
\usepackage{caption}
\usepackage{subcaption}
\usepackage{float}
\usepackage{pifont}

\usepackage{enumitem}
\setlist{leftmargin=*, itemsep=0.25em, topsep=0.25em}
\usepackage[most]{tcolorbox}
\tcbset{boxrule=0.6pt,arc=2pt}
\newtcolorbox{resultbox}{
colback=gray!4,
colframe=black!55,
boxrule=0.5pt,
arc=2pt,
left=4pt,
right=4pt,
top=3pt,
bottom=3pt,
before skip=8pt,
after skip=8pt
}
\usepackage[numbers,sort&compress]{natbib}
\definecolor{linkblue}{HTML}{0A66C2}
\usepackage{hyperref}
\hypersetup{
colorlinks=true,
linkcolor=linkblue,
citecolor=linkblue,
urlcolor=linkblue,
filecolor=linkblue
}
\usepackage[capitalize,noabbrev]{cleveref}
\Crefname{equation}{Eq.}{Eqs.}
\IfFileExists{orcidlink.sty}{\usepackage{orcidlink}}{%
\newcommand{\orcidlink}[1]{\href{https://orcid.org/#1}{\textsuperscript{\tiny
ORCID}}}
}
\usepackage{authblk}

\theoremstyle{plain}
\newtheorem{theorem}{Theorem}
\newtheorem{proposition}{Proposition}

\theoremstyle{definition}
\newtheorem{definition}{Definition}
\newtheorem{assumption}{Assumption}
\theoremstyle{remark}

\newcommand{\X}{\mathcal X}
\newcommand{\Y}{\mathcal Y}

\newcommand{\Ecal}{\mathcal E}

\newcommand{\R}{\mathbb R}
\newcommand{\Prob}{\mathbb P}
\newcommand{\Ex}{\mathbb E}
\newcommand{\1}[1]{\mathds{1}\{#1\}}
\newcommand{\perpact}{\perp}
\newcommand{\Yperp}{\Y^{\perp}}
\newcommand{\sigm}{\operatorname{sigm}}
\newcommand{\softmax}{\operatorname{softmax}}

\newcommand{\sg}{\operatorname{sg}}
\newcommand{\Hent}{\operatorname{H}}

\newcommand{\calR}{\mathcal R}
\newcommand{\calL}{\mathcal L}
\newcommand{\calS}{\mathcal S}

\newcommand{\best}[1]{\bm{#1}}
\newcommand{\second}[1]{\underline{#1}}
\DeclareRobustCommand{\RACER}{\textsc{Racer}}
\DeclareRobustCommand{\RACERKNN}{\textsc{Racer}\mbox{-KNN}}
\DeclareRobustCommand{\RACERKernel}{\textsc{Racer}\mbox{-kernel}}
\DeclareRobustCommand{\RACERC}{\textsc{Racer}\mbox{-C}}
\DeclareRobustCommand{\RACERDeepSets}{\textsc{Racer}\mbox{-DeepSets}}

\DeclareMathOperator*{\argmax}{arg\,max}
\DeclareMathOperator*{\argmin}{arg\,min}
 \newcommand{\ind}{\perp\!\!\!\!\!\!\perp} 
\newcommand{\maybeincludegraphics}[2]{%
\IfFileExists{#1}{%
\includegraphics[width=#2]{#1}%
}{%
\fbox{%
\parbox[c][0.20\textheight][c]{#2}{%
\centering
Placeholder figure:\\
\texttt{\detokenize{#1}}%
}%
}%
}%
}
\newif\ifblindreview
\blindreviewfalse
\title{\RACER{}: Role-Aligned Competence Estimation for Human-AI Routing}
\ifblindreview
\author{\textbf{Anonymous authors}}
\affil{}
\else
\author[1*]{Joshua Strong\orcidlink{0009-0009-2348-1955}}
\author[1]{Emma Sun\orcidlink{0009-0001-3816-5391}}
\author[1]{Alexander Capstick\orcidlink{0000-0002-3884-8045}}
\author[1]{Pramit Saha\orcidlink{0009-0006-8090-1969}}
\author[1]{Cheng Ouyang}
\author[1]{J.\ Alison Noble\orcidlink{0000-0002-3060-3772}}
\affil[1]{Department of Engineering Science, University of Oxford, UK}
\affil[*]{Correspondence: \texttt{joshua.strong@eng.ox.ac.uk}}
\fi
\date{\vspace{-8pt}\small Preprint. Work in Progress --- \today}
\newcommand{\keywords}[1]{%
\vspace{0.5em}\noindent\textbf{Keywords:} #1
}
\begin{document}
\maketitle
\begin{center}\small
\textbf{Work in progress.} This preprint reports the current method and
preliminary evaluation. Conclusions are restricted to the protocols and
baselines reported in this version.
\end{center}
\begin{abstract}
\noindent
Learning to defer asks a predictive system when to act autonomously and
when to defer to a human expert. Population-adaptive deferral extends this
problem to unseen experts using a small context set of expert behavior.
Neural context encoders such as L2D-Pop can be query-dependent, but may
learn routing shortcuts tied to absolute class coordinates. Identity-Free
Deferral (IFD) removes such shortcuts through role-indexed classwise
competence profiles, but its estimates are constant within each class and
cannot capture instance-level expert specialization.
We propose \RACER{}---\textbf{R}ole-\textbf{A}ligned
\textbf{C}ompetence \textbf{E}stimation for \textbf{R}outing---a
role-relative framework for estimating an unseen expert's competence
from context. \RACER{} estimates the posterior-predictive probability
that the expert is correct on a query under each candidate class role,
then combines these estimates with the model posterior to obtain the
Bayes-relevant expert-correctness probability. Nonparametric and neural
kernel-pooling estimators use candidate-role relations, shared aggregation,
and symmetric summaries, excluding absolute class-identity channels.
We prove coherent class-relabelling invariance, derive a Bayes-aligned
deferral surrogate, and give a plug-in regret bound relating routing
regret to classifier and competence-estimation error. On controlled
synthetic benchmarks, including a PathMNIST histopathology context-scaling
study with simulated experts, \RACER{} benefits from
additional context under hidden subtype dependence and gives the strongest aggregate performance on a separately sampled unseen-expert split in the CIFAR-100 synthetic experiments. On the radiologist and human--AI chest-radiography benchmarks
(VinDr-CXR and CheXpert), the \RACER{} family is competitive or best in
budget-swept deferral, with calibration results varying across metrics
and datasets.
\end{abstract}
\section{Introduction}

In high-stakes prediction, the relevant question is often not only \emph{what} an automated model should predict, but \emph{who} should make the final decision. Learning to defer (L2D) formalizes this question by jointly training a \emph{classifier} and a \emph{rejector}\footnote{Jointly training the classifier and rejector is the \emph{one-stage} L2D setup. In the \emph{two-stage} setup, the classifier is frozen and only the rejector is trained \cite{mao2023two, montreuil2024two, montreuil2025adversarial, montreuil2025ask}.}: on each input, the system either
predicts autonomously or defers to a human expert\footnote{In this work, we define an expert as a skilled yet potentially biased and fallible decision-maker.} \citep{madras2018predict,mozannar2020consistent,strong_2025_17843044}. Under the standard $0$--$1$ system loss, the Bayes rule compares the model's probability of
being correct with the expert's probability of being correct \cite[Eqn. 5]{mozannar2020consistent}. L2D therefore targets the accuracy of the combined human-AI system rather than model accuracy alone, deferring cases when the expert is expected to be more likely to predict correctly.
In healthcare, such decision allocation is part of broader human--AI
collaboration \citep{strong2026haic}, with applications including guided
deferral for medical-report classification \citep{strong2025guided}.

Classical L2D assumes a fixed expert, but it is plausible to assume that practical deployments may involve
rotating clinicians, annotators, or specialists who may be unavailable,
newly hired, or drawn from a different site. Population-adaptive L2D methods
address this by conditioning the rejector on a small context set of
expert behavior for the currently available expert so that the rejector may adapt at test-time \citep{tailor2024population, strong2026identity}.

The first approach by \citet{tailor2024population} involved a query-conditioned context encoder which compared
the current case (query) with similar examples in the expert's context and could,
in principle, route instance by instance. Yet this flexibility comes
with a weak inductive bias. Standard population encoders often expose
labels and expert predictions in fixed class coordinates, so a rejector
can learn policies tied to label identities rather than to transferable
expert roles. Here, a \emph{role} is a candidate true class $y$ for which
we assess the expert's competence. Describing competence relative to this
role means using relations such as whether a context label matches $y$,
without tying the estimator to a fixed class name or index.
Unseen experts can be in-distribution (ID), drawn from the same population
distribution as the training experts, or out-of-distribution (OOD), drawn
from a different population distribution.
Identity-Free Deferral (IFD) shows that this matters for
unseen and OOD experts: by replacing latent expert embeddings with
role-indexed classwise competence profiles, IFD enforces invariance to
coherent class relabellings \citep{strong2026identity}. However, IFD's limitation
is the opposite one: a classwise profile is robust, but it is constant
within each class and cannot express that an expert may be reliable on
some cases of a class and unreliable on others.

This paper asks whether one can obtain both desirable properties \emph{simultaneously}: query-dependent adaptation and identity-free transfer. We propose
\RACER{}---\textbf{R}ole-\textbf{A}ligned \textbf{C}ompetence
\textbf{E}stimation for \textbf{R}outing---a framework for estimating
the competence profile of an unseen expert from context using only role-relative
information. The central object is a posterior-predictive competence
functional
$\Gamma(x,y,C_e)\in[0,1],$
which estimates how likely the expert represented by context $C_e$ is to
be correct on query $x$ under candidate class role $y$. The expert's
overall correctness is then obtained by combining this rolewise
competence with the model posterior. Thus \RACER{} treats expert
competence as a probability-scale object to be estimated, and not as a hidden
feature inside a routing score as in L2D-Pop \citep{tailor2024population}.
At the same time, unlike IFD \citep{strong2026identity}, which estimates
role-aligned but classwise competence profiles that are constant within
each class, \RACER{} makes competence query-dependent while preserving the
same identity-free, role-relative inductive bias.

This distinction is important because the Bayes deferral rule compares
expert correctness and model correctness on a shared probability scale.
A rejector should therefore know not only that the expert looks
better than the model on the training distribution, but how likely the
expert is to be correct for this query and this role. \RACER{} estimates
this quantity with a strictly proper expert-correctness loss, such as
binary cross-entropy, whose population minimizer is the conditional
probability of expert correctness. Thus the competence head is trained
toward a calibrated probability target, while the final deferral decision
can still be optimized with a one-stage L2D surrogate loss. In this
sense, \RACER{} preserves the classifier--rejector co-adaptation benefits
of L2D training, but changes the expert interface: the system routes
through a role-aligned competence estimate rather than through an
unconstrained expert embedding. Empirically, this design yields consistent
routing gains in controlled synthetic settings and competitive or best
budget-swept performance on the completed real-expert chest-radiography
benchmarks.


\noindent Our contributions are:
\begin{itemize}[itemsep=0.08em, topsep=0.15em, parsep=0pt, partopsep=0pt]
\item We formulate \emph{instance-dependent identity-free deferral} for
unseen experts, combining query-dependent adaptation with coherent
class-relabelling invariance.
\item We introduce \RACER{}, a role-aligned competence-estimation
framework with nonparametric and neural kernel-pooling instantiations for
estimating expert correctness from context.
\item We establish the probability target theoretically: \RACER{} estimates
the posterior-predictive competence needed by the Bayes L2D comparison,
is trained with a strictly proper expert-correctness loss, and yields
coherently invariant deferral decisions.
\item We evaluate on controlled synthetic-expert benchmarks and
real-expert chest-radiography benchmarks. \RACER{} gives the strongest
aggregate routing on the nominal OOD split in the CIFAR-100 experiments, a strong
context-scaling trend in the synthetic-expert
PathMNIST histopathology study, and competitive or best budget-swept
performance on VinDr-CXR and CheXpert.
\end{itemize}

\section{Background}
\label{sec:background}

This section reviews the theory needed to motivate \RACER{}. We first
recall the Bayes comparison and augmented-softmax surrogate for standard
L2D. We then separate two notions of decision
consistency and calibrated expert-correctness estimation. Next, we move
to population-adaptive L2D, where the expert-correctness term must be
inferred from context, and discuss the representation problem caused by
absolute class coordinates. We then review coherent class relabelling and
IFD, which gives an identity-free but classwise solution. Finally, we
identify the missing probability object: a query-conditioned,
role-relative posterior-predictive competence function for unseen experts.
Appendix~\ref{app:extended-lit-review} discusses the broader literature,
including clinical collaboration and hierarchical deferral.

\subsection{Learning to defer}
\label{ssec:l2d}

Let $\X$ be an input space and $\Y\coloneqq \{1,\dots,K\}$ a finite label space.
A random example $(X,Y)$ is drawn from a distribution $P$ over
$\X\times\Y$, and an expert produces a prediction $M\in\Y$. A deferral
system consists of a classifier $h:\X\to\Y$ and a rejector
$r:\X\to\{0,1\}$, where $r(x)=0$ means the classifier predicts and
$r(x)=1$ means the system defers. Under $0$--$1$ system loss, the risk is
\begin{equation*}
\label{eq:standard-l2d-risk}
\calR(h,r)
\coloneqq 
\Ex\left[
(1-r(X))\1{h(X)\neq Y}
+
r(X)\1{M\neq Y}
\right].
\end{equation*}
Writing $\eta_y(x)\coloneqq \Prob(Y=y\mid X=x)$, the Bayes classifier and
rejector are \citep{mozannar2020consistent}
\begin{equation}
\label{eq:standard-bayes}
h^\star(x)=\argmax\nolimits_{y\in\Y}\eta_y(x),
\qquad
r^\star(x)
=
\1{
\Prob(M=Y\mid X=x)
\ge
\max\nolimits_{y\in\Y}\eta_y(x)
}.
\end{equation}
Thus L2D is not selective classification with a fixed rejection cost. It
is a comparison between two decision makers: the classifier and the expert.

A standard optimization strategy introduces the augmented action space
$\Yperp\coloneqq \Y\cup\{\perpact\}$, where $\perpact$ denotes deferral. Let
$g_y(x)$ be class logits and $g_\perp(x)$ be a deferral logit, and define
\[
\Pi_a(x)
\coloneqq 
\frac{\exp(g_a(x))}
{\exp(g_\perp(x))+\sum_{k=1}^K\exp(g_k(x))},
\qquad a\in\Yperp.
\]
The cross-entropy L2D surrogate of \citet{mozannar2020consistent} is
\begin{equation}
\label{eq:lce-background}
\calL_{\mathrm{CE}}(x,y,m)
\coloneqq 
-\log \Pi_y(x)
-
\1{m=y}\log \Pi_\perp(x).
\end{equation}
The first term rewards the correct class, while the second rewards
deferral when the training expert is correct. Through the shared softmax
normalization, this objective jointly trains the classifier and rejector
and can encourage a capacity-limited classifier to complement the expert
observed during training \cite[Sec.~5.1 and Fig.~1]{mozannar2020consistent}.
This training-time specialization does not imply that the classifier
adapts to an unseen expert at deployment.

The consistency of this surrogate can be seen from its pointwise
conditional risk. The full population risk is the expectation of
\eqref{eq:lce-background} over $(X,Y,M)$, but by the tower property it
decomposes as an expectation over $X$ of conditional losses. For fixed
$x$, the relevant conditional population loss is
\begin{equation*}
\label{eq:conditional-softmax-loss}
\calS(\Pi\mid x)
\coloneqq 
-\sum\nolimits_y \eta_y(x)\log \Pi_y(x)
-
q(x)\log \Pi_\perp(x),
\qquad
q(x)\coloneqq \Prob(M=Y\mid X=x).
\end{equation*}
Here $\Pi(\cdot\mid x)$ ranges over distributions on $\Yperp$. Let
$\Pi^\star(\cdot\mid x)
\in
\argmin_{\rho\in\Delta(\Yperp)}
\calS(\rho\mid x)$ denote a conditional population minimizer. A weighted
cross-entropy calculation gives
\begin{equation}
\label{eq:conditional-softmax-minimizer}
\Pi_y^\star(x)=\frac{\eta_y(x)}{1+q(x)},
\qquad
\Pi_\perp^\star(x)=\frac{q(x)}{1+q(x)}.
\end{equation}
Consequently,
\[
\Pi_\perp^\star(x)\ge \max_y\Pi_y^\star(x)
\quad\Longleftrightarrow\quad
q(x)\ge \max_y\eta_y(x),
\]
which is precisely the Bayes L2D comparison in
\eqref{eq:standard-bayes}. The augmented softmax surrogate is therefore
decision-consistent, or Bayes consistent in the sense of recovering the
optimal decision boundary between prediction and deferral \citep{bartlett2006convexity}.

\subsection{From decision consistency to expert-correctness estimation}
\label{ssec:calibration-related-work}

Decision consistency does not by itself make the augmented deferral
coordinate an expert-correctness probability. Eq. ~\ref{eq:conditional-softmax-minimizer}
shows that, at the conditional population optimum,
\[
\Pi_\perp^\star(x)=\frac{q(x)}{1+q(x)},
\qquad
q(x)=\Prob(M=Y\mid X=x).
\]
Thus the augmented deferral coordinate is a normalized decision variable,
but not the expert-correctness probability itself. Recovering $q(x)$ from
the ideal optimum requires the odds transform
\[
q(x)=\frac{\Pi_\perp^\star(x)}{1-\Pi_\perp^\star(x)}.
\]
This distinction is harmless if the goal is only the Bayes defer/predict
decision, but it matters whenever the system needs probabilities for reasons such as calibrated triage, threshold selection, budgeted routing, or comparison
between multiple experts.

Recent calibrated L2D work studies this issue in fixed-expert and fixed
multi-expert settings. \citet{verma2022calibrated} show that the standard
augmented-softmax coordinate should not be read directly as a calibrated
probability of expert correctness, and propose a one-vs-all formulation
that parameterizes expert correctness more explicitly.
\citet{verma2023learning} extend related ideas to fixed multi-expert L2D,
where expert-specific probabilities must be comparable across experts.
\citet{cao2023defense} further clarify that the issue is not the use of
softmax as such, but the probability structure as class probabilities live
on a simplex, whereas expert correctness is a separate scalar in
$[0,1]$.

For the present paper, the lesson is that the expert side of the L2D
comparison should be treated as a probability-scale object, and not just as
an augmented action coordinate. In the population-adaptive setting, this
probability must additionally be inferred from the context set of an unseen
expert.

\subsection{Population-adaptive L2D and L2D-Pop}
\label{ssec:population-adaptive}

The preceding discussion assumes a fixed expert. Population-adaptive L2D extends this setting by assuming that experts are drawn from a population
and that the expert available at test time may be unseen. The rejector must
therefore depend on both the query input and the competence of the available
expert, inferred from context. At deployment the system observes a small
context set
$C_e\coloneqq \{(x_i^C,y_i^C,m_i^C)\}_{i=1}^B$, where $B$ is the number
of context examples and $m_i^C$ is expert $e$'s
prediction on context input $x_i^C$. The population risk in this setup is
\begin{equation*}
\label{eq:pop-risk-background}
\calR_{\mathrm{pop}}(h,r)
\coloneqq 
\Ex\left[
(1-r(X,E))\1{h(X)\neq Y}
+
r(X,E)\1{M_E\neq Y}
\right],
\end{equation*}
where $r:\X\times\Ecal\to\{0,1\}$ is an expert-conditioned rejector.
The corresponding identity-conditioned Bayes rule is
\begin{equation}
\label{eq:pop-bayes-background}
r^\star(x,e)
=
\1{
\Prob(M_e=Y\mid X=x,E=e)
\ge
\max\nolimits_y \eta_y(x)
}.
\end{equation}
Thus the model posterior is still query-dependent through $x$, but the
expert-correctness term is now expert-dependent and must be inferred from
the information available for the test-time expert.

In the main framework developed in Section~\ref{sec:methodology}, the
image-only classifier is fixed at deployment, while context adapts
competence estimates and routing. We also consider an extension that
conditions the class posterior on context. Neither approach requires
test-time parameter updates.

L2D-Pop addresses this problem by learning a context encoder
\citep{tailor2024population}. Let $\omega$ denote the parameters of the image
encoder and classifier head. Write $\varphi_\omega:\X\to\R^d$ for the
image encoder, $\psi$ for the context encoder, and $d_\vartheta$ for the
deferral head that maps query features and an expert representation to a
scalar deferral logit. In the query-independent (QI) and query-conditioned (QC)
variants, respectively, L2D-Pop forms $\psi_e\coloneqq \psi(C_e)$ or
$\psi_e(x)\coloneqq \psi(x,C_e)$, and implements the expert-conditioned deferral
logit as $g_\perp(x,e)\coloneqq d_\vartheta(\varphi_\omega(x),\psi_e)$ or
$g_\perp(x,e)\coloneqq d_\vartheta(\varphi_\omega(x),\psi_e(x))$. The
query-conditioned version is especially appealing because it can attend
to context examples that are similar to the current query. L2D-Pop is
therefore a strong utility baseline because it extends one-stage L2D training
to unseen experts while preserving the classifier--rejector co-adaptation
provided by the augmented softmax objective. Our comparisons use this
softmax variant. The original L2D-Pop paper also develops a one-vs-all
surrogate and reports experiments with it \citep[Apps.~A.3 and D.1]{tailor2024population};
that variant is not evaluated here.

With sufficient data, model capacity, and successful optimization, a
context-conditioned rejector can recover the Bayes comparison in
Eq. \eqref{eq:pop-bayes-background}. Its learned routing score, however,
need not provide an explicit estimate of the expert's probability of being
correct. The softmax
L2D-Pop variant considered here learns a latent routing
margin,
\[
\Delta_{\mathrm{Pop}}(x,C_e)
\coloneqq 
g_\perp(x,e)-\max\nolimits_y g_y(x),
\]
whose purpose is to compare expert and model reliability. This margin is
an object for a defer/predict decision. The softmax variant does not
directly parameterize expert correctness as a bounded probability, although
the ideal augmented-softmax output recovers it through the odds transform
above. It also does not explicitly estimate the rolewise competence law
$\Prob(M_E=y\mid X=x,Y=y,C_E=C_e)$. Consequently, a good empirical routing
margin can be supported by several signals that are correlated on the
training population: expert competence, model uncertainty, class
frequency, context-set composition, acquisition or site artifacts, or
shortcuts tied to absolute class coordinates. These signals may improve
ID routing performance, but they need not remain aligned under expert
shift.

This issue is amplified by the way standard population encoders represent
class-indexed context information. A typical context token has the form
$[\varphi(x_i^C),e_{\mathrm{lab}}(y_i^C),e_{\mathrm{lab}}(m_i^C)]$,
where $e_{\mathrm{lab}}(\cdot)$ is a one-hot or learned label embedding.
Such tokens expose absolute label identities to the expert encoder and
rejector. They allow the model to learn coordinate-specific rules, such
as deferring when an expert appears strong on a particular class index,
rather than forcing the transferable rule: defer when the expert is
strong on the relevant role. This is the gap addressed by identity-free
approaches \citep{strong2026identity}.

\subsection{Coherent class relabelling}
\label{ssec:coherent-relabeling}

The identity issue can be formalized through a symmetry. Class names are
arbitrary: renaming them coherently should rename the classifier output,
but should not change the binary defer/predict decision. Let
$\mathfrak S_K$ be the symmetric group on $\Y$. For a permutation
$\pi\in\mathfrak S_K$, define the relabelled posterior
$\eta_y^\pi(x)\coloneqq \eta_{\pi^{-1}(y)}(x)$ and the relabelled context
\[
\pi C_e
\coloneqq 
\{(x_i^C,\pi(y_i^C),\pi(m_i^C))\}_{i=1}^B .
\]

\begin{definition}[Coherent class relabelling]
\label{def:cr}
A coherent class relabelling applies the same permutation
$\pi\in\mathfrak S_K$ to every class-indexed object: labels, expert
predictions, model posteriors, and candidate roles. A classifier is
\emph{coherently equivariant} if $h^\pi(x)=\pi(h(x))$. A rejector is
\emph{coherently invariant} if $r^\pi(x,\pi C_e)=r(x,C_e)$.
\end{definition}

The Bayes classifier is equivariant under coherent class relabelling, and
the Bayes rejector is invariant \cite[Propn.~3.1]{strong2026identity}.
Architectures that expose absolute class coordinates need not satisfy
these symmetry conditions. This is a distribution-free
architectural vulnerability: even if such models perform well on ID
experts, their hypothesis class permits identity-conditioned shortcuts
that are not stable under expert shifts.

\subsection{Identity-Free Deferral}
\label{ssec:ifd-background}

IFD addresses this vulnerability by replacing unstructured expert
embeddings with explicit Bayesian competence profiles
\citep{strong2026identity}. For expert $e$, class $y$, and context
counts
\[
n_{e,y}\coloneqq \sum\nolimits_i\1{y_i^C=y},
\qquad
t_{e,y}\coloneqq \sum\nolimits_i\1{y_i^C=y,\ m_i^C=y_i^C},
\]
IFD estimates the classwise correctness probability
\[
\theta_{e,y}
\approx
\Prob(M_e=Y\mid Y=y,C_e).
\]
The rejector does not receive this profile as a labelled vector in fixed
coordinates. Instead, it receives role-indexed quantities, such as the
posterior mean and uncertainty at the model's top-ranked class, the
expert's best estimated class, and symmetric summaries across classes.
This is the key strength of IFD: it is data-efficient, interpretable,
and coherently invariant by construction. 

In this sense, IFD is closer to probability-scale competence estimation
than a latent deferral-margin method: its profile entries are explicit
posterior estimates of classwise expert correctness. Its limitation is
therefore not that competence is hidden inside an unconstrained routing
score, but that the estimated probability is classwise rather than
query-conditioned.

IFD approximates the expert's rolewise competence by
\begin{equation}
\label{eq:ifd-background-gamma-restriction}
\Gamma_{\mathrm{IFD}}(x,y,C_e)\coloneqq \theta_{e,y},
\end{equation}
which removes the dependence on query $x$. This is appropriate when
expert ability is mostly classwise. It is biased when performance varies
within class, for example when an expert is reliable on prototypical
cases but not atypical ones, or when image quality, acquisition site,
subtype, or local morphology changes the expert's reliability without
changing the label.

\subsection{The missing probability object}
\label{ssec:probability-object-view}

\begin{table}[h]
\centering
\caption{Main conceptual comparison. L2D-Pop is expressive but learns a
latent routing margin; IFD is role-aligned but classwise. \RACER{}
estimates the query-conditioned role competence object whose
marginalization gives expert correctness.}
\label{tab:comparison}
\small
\setlength{\tabcolsep}{4pt}
\renewcommand{\arraystretch}{1.08}
\begin{tabularx}{\linewidth}{@{}
>{\raggedright\arraybackslash}p{0.17\linewidth}
>{\raggedright\arraybackslash}p{0.22\linewidth}
>{\raggedright\arraybackslash}p{0.25\linewidth}
>{\raggedright\arraybackslash}X
@{}}
\toprule
Method & Context interface & Learned or estimated object & Strength and limitation \\
\midrule
L2D-Pop QI/QC (softmax) \citep{tailor2024population}
& Latent embedding $\psi(C_e)$ or $\psi(x,C_e)$
& Routing margin $g_\perp(x,\psi)-\max_y g_y(x)$
& Strong one-stage routing and, for QC, instance dependence; not explicitly calibrated as expert correctness and not generally identity-free \\
IFD \citep{strong2026identity}
& Bayesian classwise profile read through role-indexed summaries
& $\Gamma_{\mathrm{IFD}}(x,y,C_e)=\theta_{e,y}$
& Explicit probability-scale competence and coherent invariance; query-independent within class \\
\rowcolor{gray!10}
\RACER{} \emph{(Ours)}
& Role-relative query--context competence estimator
& $\Gamma_{\RACER{}}(x,y,C_e)\approx\Prob(M_E=y\mid X=x,Y=y,C_E=C_e)$
& Instance-dependent, probability-scale, and identity-free; requires informative context geometry and sufficient support \\
\bottomrule
\end{tabularx}
\end{table}

The preceding subsections identify two separate requirements for
population-adaptive deferral. From calibrated L2D, the expert side of the
Bayes comparison should be a probability-scale correctness estimate.
From coherent relabelling and IFD, this estimate should be expressed in
role-relative rather than absolute class coordinates. What remains
missing is the instance-dependent version of this object for an unseen
expert represented only by context.

For an unseen expert, the Bayes population rejector with the information
available at deployment requires the posterior-predictive correctness
probability
\begin{equation*}
\label{eq:q-object-view}
q^\star(x,C_e)\coloneqq \Prob(M_E=Y\mid X=x,C_E=C_e).
\end{equation*}
Conditioning on the unknown true class gives via the law of total probability
\[
q^\star(x,C_e)
=
\sum\nolimits_{y\in\Y}
\Prob(Y=y\mid X=x,C_E=C_e)
\Prob(M_E=y\mid X=x,Y=y,C_E=C_e).
\]
Under the context-posterior invariance condition formalized in
\Cref{assump:context-posterior-invariance}, the expert context does not alter
our label posterior once $x$ is known, i.e. $Y\ind C_E\mid X$. This should be
understood as a context-acquisition condition rather than as a property of every
possible historical expert record. Appendix~\ref{app:context-posterior-extension}
describes a more general context-conditioned posterior variant. Under this
condition, the decomposition becomes
\begin{equation*}
\label{eq:gamma-object-view}
q^\star(x,C_e)
=
\sum_{y\in\Y}
\eta_y(x)
\underbrace{
\Prob(M_E=y\mid X=x,Y=y,C_E=C_e)
}_{\Gamma^\star(x,y,C_e)}.
\end{equation*}
Thus the central population-adaptive competence object is
\[
\Gamma^\star(x,y,C_e)
\coloneqq 
\Prob(M_E=y\mid X=x,Y=y,C_E=C_e).
\]
This posterior-predictive quantity averages over uncertainty about the
expert given the labelled context examples in $C_e$. When the expert
identity $E=e$ is known, the corresponding oracle target is
$\Prob(M_E=y\mid X=x,Y=y,E=e)$.

This view clarifies the relationship between existing methods. L2D-Pop
can learn the Bayes routing margin implicitly. Its softmax variant does not
directly parameterize $q^\star(x,C_e)$ as a bounded probability or explicitly
estimate the rolewise competence law $\Gamma^\star(x,y,C_e)$. IFD estimates competence explicitly and
identity-free, but uses the query-independent approximation
$\Gamma_{\mathrm{IFD}}(x,y,C_e)=\theta_{e,y}$. \RACER{} is designed to
estimate the missing object directly: a query-conditioned,
role-relative, probability-scale competence function whose
marginalization against the model posterior gives expert correctness. \Cref{tab:comparison} summarizes our positioning: L2D-Pop provides
expressive context-conditioned routing but leaves expert competence
implicit; IFD estimates competence explicitly and identity-free but only
classwise; \RACER{} targets the missing query-conditioned,
role-relative competence probability.

\section{Methodology}
\label{sec:methodology}

\RACER{} separates classification, expert-competence estimation, and routing. The classifier
produces an estimate $p_\omega(y\mid x)$ of the label posterior
$\eta_y(x)$, the competence module estimates how
likely the expert is to be correct on a candidate role, and the final
deferral rule compares the induced expert-correctness probability with
the classifier's confidence in its predicted class, $\max_y p_\omega(y\mid x)$.

We study two competence estimators. \RACERKNN{} uses a nonparametric
$k$-nearest-neighbour estimate of expert correctness within each candidate
class. \RACERKernel{}, our default implementation, applies a learned
competence network to kernel-weighted context summaries and classifier
statistics. Both use context to estimate expert competence while retaining
an image-only class posterior. Section~\ref{ssec:racer-estimators} gives
their full specifications.

This section defines the
competence target, describes the role-relative estimators, and proves
the two properties needed for transfer: probability-scale expert-correctness
estimation and coherent-relabelling invariance.

\subsection{Problem setup and competence target}
\label{ssec:problem-setup}

We consider a population of experts. An expert identity $E\in\Ecal$ is
drawn from a population distribution. At deployment, expert $e$ may be
unseen. The system observes a context set of $B$ labelled examples of
that expert's predictions,
\[
C_E=C_e=\{(x_i^C,y_i^C,m_i^C)\}_{i=1}^B,
\]
and then receives a query $x$. We write
$a_i^C\coloneqq \1{m_i^C=y_i^C}$ for context correctness. The system must either
output a class or defer to the expert represented by $C_e$.

Recall that $\eta_y(x)$ denotes the true label posterior. Let
$p_\omega(y\mid x)$ be the classifier's estimated posterior, and
$p_{\omega,\max}(x)\coloneqq \max_y p_\omega(y\mid x)$. If the expert identity
were known, the oracle expert-correctness term would be
\begin{equation*}
\label{eq:oracle-expert-correctness}
q_e^{\mathrm{or}}(x)\coloneqq \Prob(M_E=Y\mid X=x,E=e).
\end{equation*}
\RACER{} does not assume access to this identity-level oracle. Its
deployable target is the posterior-predictive correctness probability
conditioned on the context available at test time,
\begin{equation*}
\label{eq:posterior-predictive-q}
q^\star(x,C_e)=\Prob(M_E=Y\mid X=x,C_E=C_e).
\end{equation*}
The corresponding rolewise competence target is
\begin{equation*}
\label{eq:gamma-functional}
\Gamma^\star(x,y,C_e)
=
\Prob(M_E=y\mid X=x,Y=y,C_E=C_e).
\end{equation*}
\begin{assumption}[Context-posterior invariance]
\label{assump:context-posterior-invariance}
For the deployment contexts considered in this paper,
\[
\Prob(Y=y\mid X=x,C_E=C_e)=\Prob(Y=y\mid X=x)=\eta_y(x), \quad \forall x,y \text{ and feasible contexts } C_e.
\]
\end{assumption}

\Cref{assump:context-posterior-invariance} should be read as a
context-acquisition condition rather than a property of arbitrary historical
logs. It is appropriate for randomized or stratified warm-up contexts sampled
independently of the future query stream, where $C_e$ reveals expert behavior
but not additional case-mix information. Historical contexts tied to site,
scanner, ward, service line, referral pathway, patient cohort, or other
prevalence-shifting workflow variables may violate it. In such deployments, the
Bayes expert-correctness term becomes
\[
q^\star(x,C_e)
=
\sum\nolimits_{y\in\Y}
\Prob(Y=y\mid X=x,C_E=C_e)\Gamma^\star(x,y,C_e),
\]
and we consider an optional extension, \RACERC{}, where C denotes
\emph{context-conditioned classification}. It retains \RACERKernel{}'s
competence-estimation architecture and adds a role-equivariant posterior
adapter with parameters $\rho$, replacing $p_\omega(y\mid x)$ with
$\pi_\rho(y\mid x,C_e)$. The adapted posterior is used both to combine
rolewise competence estimates and to determine the classifier's prediction
and confidence. Appendix~\ref{app:context-posterior-extension} gives the
construction, and Appendix~\ref{app:context-posterior-stress} evaluates it
under controlled shifts in case mix. Under
\Cref{assump:context-posterior-invariance}, we retain \RACERKernel{} as the
default, avoiding an additional posterior-estimation task that can
increase variance when context is limited.

Under \Cref{assump:context-posterior-invariance},
the law of total probability gives
\begin{equation*}
\label{eq:q-star-decomposition}
q^\star(x,C_e)
=
\sum\nolimits_{y\in\Y}\eta_y(x)\Gamma^\star(x,y,C_e).
\end{equation*}

\RACER{} estimates $\Gamma^\star$ with a role-aligned competence
functional $\Gamma(x,y,C_e)\in[0,1]$. The induced expert-correctness
estimate and plug-in rejector are
\begin{equation}
\label{eq:qhat-plugin-rejector}
\widehat q_\omega(x,C_e)
\coloneqq 
\sum\nolimits_{y\in\Y}p_\omega(y\mid x)\Gamma(x,y,C_e),
\qquad
\widehat r_\tau(x,C_e)
\coloneqq 
\1{\widehat q_\omega(x,C_e)-p_{\omega,\max}(x)\ge \tau}.
\end{equation}
Here $\tau$ controls the deferral budget or cost. The rest of the
methodology answers two questions: (i) \emph{how should $\Gamma$ be
parameterized so that it is expressive but identity-free?}, and (ii) \emph{how
should it be trained so that it remains a probability-scale competence
estimate rather than only a routing feature?}

\subsection{Role-relative information}
\label{ssec:role-relative-info}

Let $\varphi_\omega:\X\to\R^d$ be an image encoder, with
$z\coloneqq \varphi_\omega(x)$ and $z_i^C\coloneqq \varphi_\omega(x_i^C)$. For a candidate
role $y$, define the posterior rank
$R_\omega(y;x)\coloneqq 1+\sum_{k\in\Y}\1{p_\omega(k\mid x)>p_\omega(y\mid x)}$.
\RACER{} allows class information to enter the competence estimator only
through relations to the candidate role. For query $x$, role $y$, and
context item $i$, admissible quantities include
\begin{align*}
\label{eq:role-relative-features}
\underbrace{z,\ z_i^C,\ s(z,z_i^C)}_{\text{query--context geometry}},
\quad
\underbrace{a_i^C}_{\text{context correctness}},
\quad
\underbrace{\1{y_i^C=y},\ \1{m_i^C=y}}_{\text{candidate-role relations}},
\nonumber\\[-2pt]
\underbrace{p_\omega(y\mid x),\ R_\omega(y;x)}_{\text{query role state}},
\quad
\underbrace{p_\omega(y_i^C\mid x_i^C),\ R_\omega(y_i^C;x_i^C)}_{\text{context role state}},
\quad
\underbrace{\Hent(p_\omega(\cdot\mid x))}_{\text{query uncertainty}}.
\end{align*}
The precise feature list is less important than the invariance rule.
Labels may appear through equality relations, ranks, posterior values at
roles, or correctness indicators, but not through learned label
embeddings, untied class-specific channels, or fixed class-coordinate
vectors. Image features are allowed as non-class-indexed geometry. Any
class-indexed object derived from the classifier must transform
equivariantly under coherent relabelling (\cref{def:cr}).

Formally, let $T$ denote a summary map used by a competence estimator. We
call $T$ role-admissible if, for every class permutation
$\pi\in\mathfrak S_K$ and every posterior satisfying
$p_\omega^\pi(\pi(y)\mid x)=p_\omega(y\mid x)$,
\begin{equation*}
\label{eq:role-admissible-summary}
T(x,\pi(y),\pi C_e,p_\omega^\pi)=T(x,y,C_e,p_\omega).
\end{equation*}
\RACER{} instantiations use only role-admissible summaries and apply the
same functions across roles. This is the architectural difference
between a role-aligned competence estimator and a generic population
encoder.

\subsection{Same-role kernel pooling}
\label{ssec:same-role-kernel}

Both \RACERKNN{} and \RACERKernel{} use a query-conditioned same-role
pooling primitive. Let
$\tilde z\coloneqq \varphi_\omega(x)/\|\varphi_\omega(x)\|_2$ and
$\tilde z_i^C\coloneqq \varphi_\omega(x_i^C)/\|\varphi_\omega(x_i^C)\|_2$, define
the cosine similarity $s_i(x)\coloneqq \tilde z^\top\tilde z_i^C$, and set
$K_i(x)\coloneqq \exp(s_i(x)/\tau_{\mathrm{ker}})$. For candidate role $y$, the
same-role similarity mass is
\begin{equation}
\label{eq:same-role-mass}
S_{e,y}(x)\coloneqq \sum\nolimits_{i=1}^B\1{y_i^C=y}K_i(x),
\qquad
N_{e,y}\coloneqq \sum\nolimits_{i=1}^B\1{y_i^C=y}.
\end{equation}
These quantities compare the query only to context examples whose true
label matches the candidate role $y$.

A normalized local average is undefined without same-role support. We
define a local competence statistic with an explicit fallback and optional
prior smoothing. Let $\mu_{e,0}\in[0,1]$ be a global-context prior,
for example
\begin{equation*}
\label{eq:global-context-prior}
\mu_{e,0}
\coloneqq 
\frac{\alpha_{\mathrm{pop}}\mu_{\mathrm{pop}}+
\sum_{i=1}^B a_i^C}{\alpha_{\mathrm{pop}}+B},
\end{equation*}
where $\mu_{\mathrm{pop}}$ is a population-level expert-accuracy prior
and $\alpha_{\mathrm{pop}}\ge0$ is its strength, with
$\alpha_{\mathrm{pop}}+B>0$. For local prior strength
$\alpha_0\ge0$, define
\begin{equation}
\label{eq:local-smoothed-statistic}
\bar a_{e,y}(x)
\coloneqq 
\begin{cases}
\displaystyle\frac{\alpha_0\mu_{e,0}+
\sum_{i=1}^B\1{y_i^C=y}K_i(x)a_i^C}{\alpha_0+S_{e,y}(x)},
& \alpha_0+S_{e,y}(x)>0,\\[6pt]
\mu_{e,0}, & \alpha_0+S_{e,y}(x)=0.
\end{cases}
\end{equation}
Thus $\bar a_{e,y}(x)$ is a similarity-weighted empirical correctness
estimate for expert $e$ on context examples that are both similar to $x$
and have candidate role $y$, shrunk toward a global context prior when
same-role support is weak and $\alpha_0>0$. If $S_{e,y}(x)=0$, the estimate equals
$\mu_{e,0}$. It is query-dependent through $s_i(x)$ and identity-free
because the candidate class appears only through the equality mask
$\1{y_i^C=y}$.

In the supplied implementation, $\alpha_0=0$: supported roles use a
normalized local mean, while unsupported roles fall back to global context
accuracy. The global prior is the empirical context accuracy for $B>0$, with
$1/2$ used for empty context. Kernel weights are computed with a stable
softmax; the auxiliary similarity-mass feature caps exponential arguments at
30 for numerical stability.

\subsection{Two competence estimators}
\label{ssec:racer-estimators}

\paragraph{Nonparametric estimator.}
\RACERKNN{} uses the local same-role statistic with uniform weights on
the nearest same-role neighbors (and zero weight on the remaining items):
\begin{equation*}
\label{eq:racer-knn}
\Gamma_{\mathrm{KNN}}(x,y,C_e)\coloneqq \bar a_{e,y}(x).
\end{equation*}
This estimator is transparent in a case-based sense: its competence
estimate is the empirical correctness of expert $e$ on nearby same-role
context examples. The implemented KNN variant uses a uniform average over
the $\min(k,N_{e,y})$ nearest same-role items, rather than exponential weights,
and uses global context accuracy when no same-role item is available.

\paragraph{Learned kernel estimator.}
\RACERKernel{} adds a role-shared competence network on top of the
normalized kernel statistic. Its learned mapping supplies smoothing beyond
the unsupported-role fallback. Let
\[
\mathrm{margin}_\omega(y;x)
\coloneqq p_\omega(y\mid x)-\max_{k\neq y}p_\omega(k\mid x)
\]
be the posterior margin of role $y$. The role-relative summary is
\begin{equation*}
\label{eq:racer-kernel-summary}
\begin{aligned}
u_{e,y}(x)\coloneqq \Big(&\bar a_{e,y}(x),\ \log(1+N_{e,y}),\ \log(1+S_{e,y}(x)),\\
&p_\omega(y\mid x),\ R_\omega(y;x),\ \mathrm{margin}_\omega(y;x),\
\Hent(p_\omega(\cdot\mid x)),\ \mu_{e,0}\Big).
\end{aligned}
\end{equation*}
A small MLP $g_\theta$ is shared across all roles and defines
\begin{equation*}
\label{eq:racer-kernel-gamma}
\Gamma_\theta(x,y,C_e)
\coloneqq 
\sigm\!\left(g_\theta(u_{e,y}(x))\right).
\end{equation*}
Because the same network is applied to every role and all inputs are
role-relative, \RACERKernel{} can learn nonlinear smoothing and
calibration without learning class-identity-specific rules. In our
probability-scale interpretation, $u_{e,y}(x)$ is a summary of the full
information $(x,y,C_e)$; the proper loss below targets conditional correctness
given that summary at its population optimum.

\subsection{Proper competence training and plug-in calibration}
\label{ssec:calibration-target}

For a target tuple $(x,y,m)$ from expert $e$ and context $C_e$, define
$a\coloneqq \1{m=y}$. The neural competence head is trained with the binary
proper scoring rule
\begin{equation}
\label{eq:racer-competence-bce}
\calL_{\mathrm{comp}}(\theta)
\coloneqq 
\Ex\!\left[
-a\log \Gamma_\theta(x,y,C_e)
-(1-a)\log(1-\Gamma_\theta(x,y,C_e))
\right],
\end{equation}
where the expectation is over sampled experts, their contexts, and target
examples from the same expert. When the classifier posterior is intended
to remain a calibrated estimate of $\eta$, the representation
$\varphi_\omega$ and posterior $p_\omega$ used inside $u_{e,y}(x)$ are
best treated as fixed or stop-gradient inputs to the competence loss. If
end-to-end feature sharing is used instead, classifier calibration should
be rechecked after joint training.

\begin{proposition}[Proper target and summary projection]
\label{prop:proper-competence}
Let $A\coloneqq \1{M_E=Y}$ and let $U\coloneqq T(X,Y,C_E,p_\omega)$ be any summary made
available to a scalar predictor $g(U)\in[0,1]$, with BCE extended to the
boundary by continuity. The conditional minimizer
of the binary cross-entropy over all measurable functions of $U$ is
\[
g^\star(U)=\Prob(A=1\mid U).
\]
In particular, if $U=(X,Y,C_E)$, then
$g^\star(x,y,C_e)=\Gamma^\star(x,y,C_e)$. More generally,
the unrestricted population BCE optimum given the role-relative summary
$u_{e,y}(x)$ equals its conditional correctness probability; it recovers the
full competence functional when that summary is sufficient. This does not
certify finite-network optimization or calibration under distribution shift.
\end{proposition}

\begin{proof}
Condition on $U=u$ and write $\alpha\coloneqq \Prob(A=1\mid U=u)$. The
conditional risk for a scalar prediction $g\in(0,1)$ is
$-\alpha\log g-(1-\alpha)\log(1-g)$. Its derivative is
$-\alpha/g+(1-\alpha)/(1-g)$, whose unique zero is $g=\alpha$. Strict
convexity gives the unique minimizer for $0<\alpha<1$; for
$\alpha\in\{0,1\}$ the corresponding boundary value minimizes the extended loss. Taking $U=(X,Y,C_E)$ yields
$\Prob(A=1\mid X=x,Y=y,C_E=C_e)=\Gamma^\star(x,y,C_e)$. Full proof in Appendix~\ref{app:proof-proper-competence}.
\end{proof}

Under context-posterior invariance, the plug-in estimate recovers expert
correctness when the classifier posterior and rolewise competence estimates
recover their respective conditional probability targets. Marginal calibration
of the two predictors alone is insufficient.

\begin{proposition}[Plug-in calibration]
\label{prop:qhat-calibration}
Under \Cref{assump:context-posterior-invariance}, if
$p_\omega(y\mid x)=\eta_y(x)$ and
$\Gamma(x,y,C_e)=\Gamma^\star(x,y,C_e)$ for all $y$, then
\[
\widehat q_\omega(x,C_e)
=
\sum\nolimits_y p_\omega(y\mid x)\Gamma(x,y,C_e)
=
q^\star(x,C_e)
=
\Prob(M_E=Y\mid X=x,C_E=C_e).
\]
\end{proposition}

\begin{proof}
Apply the decomposition in \eqref{eq:q-star-decomposition} and substitute
the assumed conditional probability targets for $p_\omega$ and $\Gamma$.
\end{proof}

The next bound states the corresponding plug-in decision guarantee. It
is intentionally a decision-regret statement, not a claim that finite
neural training is automatically Bayes-consistent. The bound concerns ordinary
$0$--$1$ system loss without an additional deferral cost. It does not establish
optimality for macro AURSBAC, nor does zero-threshold Bayes alignment alone
guarantee optimal ranking across deferral budgets.

\begin{proposition}[Plug-in routing regret]
\label{prop:plugin-regret}
Let $h_\omega(x)\coloneqq \argmax_y p_\omega(y\mid x)$ and let
$\widehat r_0(x,C_e)=\1{\widehat q_\omega(x,C_e)\ge p_{\omega,\max}(x)}$.
Let $h^\star(x)=\argmax_y\eta_y(x)$ and
$r^\star(x,C_e)=\1{q^\star(x,C_e)\ge \eta_{\max}(x)}$, where
$\eta_{\max}(x)\coloneqq \max_y\eta_y(x)$. Under the $0$--$1$ system loss with no
additional deferral cost,
\begin{align*}
\begin{split}
&\calR(h_\omega,\widehat r_0)-\calR(h^\star,r^\star) \\ & \le
\Ex\Bigg[
3\big\|p_\omega(\cdot\mid X)-\eta(\cdot\mid X)\big\|_1 +
\sum_{y\in\Y}\eta_y(X)
\big|\Gamma(X,y,C_E)-\Gamma^\star(X,y,C_E)\big|
\Bigg].
\end{split}
\end{align*}
\end{proposition}

\begin{proof}
Condition on $(X=x,C_E=C_e)$ and abbreviate
$q\coloneqq q^\star(x,C_e)$, $\hat q\coloneqq \widehat q_\omega(x,C_e)$,
$\eta_{\max}\coloneqq \max_y\eta_y(x)$, and
$\hat p_{\max}\coloneqq p_{\omega,\max}(x)$. The Bayes conditional correctness is
$\max\{\eta_{\max},q\}$, while the plug-in system has conditional
correctness $(1-\widehat r_0)\eta_{h_\omega(x)}+\widehat r_0 q$. Add and
subtract the correctness that would be obtained by using the Bayes top
class whenever the plug-in system predicts. The resulting excess is at
most
\[
\eta_{\max}-\eta_{h_\omega(x)}
+
|q-\eta_{\max}|\,
\1{\operatorname{sign}(q-\eta_{\max})\neq
\operatorname{sign}(\hat q-\hat p_{\max})}.
\]
The first term is at most
$\|p_\omega(\cdot\mid x)-\eta(\cdot\mid x)\|_1$. On the sign-disagreement
event,
$|q-\eta_{\max}|\le |\hat q-q|+|\hat p_{\max}-\eta_{\max}|$. Moreover,
$|\hat p_{\max}-\eta_{\max}|\le
\|p_\omega(\cdot\mid x)-\eta(\cdot\mid x)\|_1$, and
\[
|\hat q-q|
\le
\|p_\omega(\cdot\mid x)-\eta(\cdot\mid x)\|_1
+
\sum\nolimits_y\eta_y(x)|\Gamma(x,y,C_e)-\Gamma^\star(x,y,C_e)|.
\]
Combining these inequalities and taking expectation gives the claim. Full proof in App. \ref{app:proof-plugin-regret}.
\end{proof}

This is the sense in which \RACER{} differs from a pure routing model.
The final rejector may be trained with a decision surrogate, but the
expert side of the comparison has a separate probability target and an
explicit path from estimation error to deferral regret.

\subsection{Invariance and instance dependence}
\label{ssec:invariance-instance}

Under a coherent relabelling $\pi\in\mathfrak S_K$, assume the classifier
posterior transforms equivariantly:
$p_\omega^\pi(\pi(y)\mid x)=p_\omega(y\mid x)$. Then the quantities used
by \RACER{} are preserved at corresponding roles: same-role masks,
correctness indicators, similarities, ranks, entropies, posterior values,
support counts, similarity masses, the global prior $\mu_{e,0}$, and the
smoothed statistic $\bar a_{e,y}(x)$.

\begin{theorem}[Coherent-relabelling invariance]
\label{thm:cr-invariance}
Under posterior equivariance, both \RACERKNN{} and \RACERKernel{} satisfy
\[
\Gamma^\pi(x,\pi(y),\pi C_e)
=
\Gamma(x,y,C_e),
\]
where $\Gamma$ denotes the corresponding estimator
$\Gamma_{\mathrm{KNN}}$ or $\Gamma_\theta$. Consequently,
$\widehat q_\omega^\pi(x,\pi C_e)=\widehat q_\omega(x,C_e)$,
$p_{\omega,\max}^\pi(x)=p_{\omega,\max}(x)$, and the plug-in rejector in
Eq. \eqref{eq:qhat-plugin-rejector} is coherently invariant.
\end{theorem}

\begin{proof}
For every context item $i$,
$\1{y_i^C=y}=\1{\pi(y_i^C)=\pi(y)}$ and
$\1{m_i^C=y_i^C}=\1{\pi(m_i^C)=\pi(y_i^C)}$. Image features and
query--context similarities are unchanged, while posterior equivariance
preserves posterior values, ranks, margins, and entropies at
corresponding roles. Therefore $N_{e,y}$, $S_{e,y}(x)$,
$\mu_{e,0}$, $\bar a_{e,y}(x)$, and the summary vector $u_{e,y}(x)$ are
unchanged after replacing $(y,C_e)$ by $(\pi(y),\pi C_e)$. The result for
\RACERKNN{} follows immediately. The result for \RACERKernel{} follows
because the same function $g_\theta$ is shared across all roles. Summing
over relabelled roles gives invariance of $\widehat q_\omega$, and the
maximum posterior is unchanged, so the binary rejector is invariant. Full proof in App. \ref{app:proof-cr-invariance}.
\end{proof}

This theorem concerns coherent relabelling of the task representation: all
labels, expert predictions, candidate roles, and class-indexed posteriors
are renamed together. It does not claim invariance to an expert-only
competence shift relative to fixed class names; in that case the Bayes
deferral action itself may change, and the context should cause the
competence estimate to change.

Identity-freeness does not imply query-independence. If two same-class
queries $x$ and $x'$ have different similarities to correct and incorrect
same-role context examples, then the kernel masses in
\eqref{eq:same-role-mass} and the smoothed statistic in
\eqref{eq:local-smoothed-statistic} can differ. Hence
$\Gamma(x,y,C_e)\neq\Gamma(x',y,C_e)$ even for the same candidate role
$y$. This is precisely the capacity missing from classwise profiles.

\subsection{Bayes-aligned deferral training}
\label{ssec:bayes-aligned-training}

The competence loss trains $\Gamma_\theta$ as a probability-scale expert
correctness estimate. To train an explicit deferral head while preserving
one-stage L2D co-adaptation, let $f_{\omega,y}(x)$ be class logits and
let $d_\vartheta(x,C_e)$ be an invariant deferral logit, for example a
function of $\widehat q_\omega(x,C_e)$, $p_{\omega,\max}(x)$, their
difference, and $\Hent(p_\omega(\cdot\mid x))$. In the implementation,
all these inputs are detached:
\[
d_\vartheta(x,C_e)\coloneqq D_\vartheta\!\left(
\sg\!\left[\widehat q_\omega,\ p_{\omega,\max},\
\widehat q_\omega-p_{\omega,\max},\ \Hent(p_\omega)\right]\right).
\]
Define
\begin{align*}
\Pi_y(x,C_e) &\coloneqq  \frac{\exp(f_{\omega,y}(x))}{\exp(d_\vartheta(x,C_e))+\sum_{k=1}^K\exp(f_{\omega,k}(x))}, \\[10pt]
\Pi_\perp(x,C_e) &\coloneqq  \frac{\exp(d_\vartheta(x,C_e))}{\exp(d_\vartheta(x,C_e))+\sum_{k=1}^K\exp(f_{\omega,k}(x))}.
\end{align*}
For a labelled example $(x,y)$, the context-only deferral loss is
\begin{equation*}
\label{eq:context-only-loss}
\calL_{\mathrm{RACER}}(x,y,C_e)
\coloneqq 
-\log \Pi_y(x,C_e)
-
\sg[\Gamma(x,y,C_e)]\log \Pi_\perp(x,C_e),
\end{equation*}
where $\sg[\cdot]$ denotes stop-gradient. Detaching both the loss weight
and the deferral-head inputs blocks routing-loss gradients to the
competence-head parameters $\theta$. The shared encoder and classifier still
receive routing gradients through the class logits, and competence BCE can
update the query representation. Thus this is parameter-level isolation of the
competence head, not a guarantee that joint training preserves its predictions
or calibration. The kernel training objective adds
$\lambda_{\mathrm{comp}}\calL_{\mathrm{comp}}$ to this routing loss.

\begin{theorem}[Bayes alignment]
\label{thm:bayes-alignment}
Let $\Gamma^\star(x,y,C_e)=\Prob(M_E=y\mid X=x,Y=y,C_E=C_e)$ and
$q^\star(x,C_e)=\sum_y\eta_y(x)\Gamma^\star(x,y,C_e)$. For fixed
$(x,C_e)$, the conditional minimizer of the idealized loss
$-\log\Pi_Y-\Gamma^\star(x,Y,C_e)\log\Pi_\perp$ over distributions on
$\Yperp$ satisfies
\[
\Pi_y^\star(x,C_e)=\frac{\eta_y(x)}{1+q^\star(x,C_e)},
\qquad
\Pi_\perp^\star(x,C_e)=\frac{q^\star(x,C_e)}{1+q^\star(x,C_e)}.
\]
Therefore $\Pi_\perp^\star(x,C_e)\ge\max_y\Pi_y^\star(x,C_e)$ if and
only if $q^\star(x,C_e)\ge\max_y\eta_y(x)$.
\end{theorem}

\begin{proof}
Condition on $(x,C_e)$. Taking expectation over $Y$ gives the conditional
risk
\[
-\sum\nolimits_y\eta_y(x)\log\Pi_y
-\left(\sum\nolimits_y\eta_y(x)\Gamma^\star(x,y,C_e)\right)\log\Pi_\perp
=
-\sum\nolimits_y\eta_y(x)\log\Pi_y-q^\star(x,C_e)\log\Pi_\perp.
\]
This is cross-entropy with nonnegative weights
$\{\eta_y(x)\}_{y\in\Y}$ and $q^\star(x,C_e)$, whose total mass is
$1+q^\star(x,C_e)$. Normalizing these weights gives the stated
minimizer. The decision equivalence follows by multiplying
$\Pi_\perp^\star\ge\max_y\Pi_y^\star$ by $1+q^\star(x,C_e)$. Full proof in \ref{app:proof-bayes-alignment}.
\end{proof}

Thus \RACER{} does not reject the Mozannar-style augmented surrogate. It
uses it at the decision layer, where it is appropriate, while the
expert-competence channel is separately constrained to estimate a
probability. The theorem is deliberately conditional: full statistical
consistency of an implemented system requires the classifier posterior,
competence estimator, model class, and optimization to approach their
population targets. \Cref{prop:plugin-regret} makes the plug-in
dependence explicit.

\subsection{Relation to existing methods}
\label{ssec:method-comparison}

\RACER{} can be viewed as interpolating between the strengths of the two
main baselines, L2D-Pop \citep{tailor2024population} and IFD
\citep{strong2026identity}. If the query-dependent kernel is removed and
$\Gamma(x,y,C_e)$ is replaced by a classwise posterior mean
$\theta_{e,y}$, the modelling assumption reduces to that of IFD. This has
low variance and exact invariance, but cannot adapt within class. If
instead the role-relative competence channel is replaced by a generic
latent context encoder trained only through the augmented softmax margin,
one obtains the L2D-Pop interface. This has high expressivity and strong
one-stage routing, but rolewise competence is not explicitly estimated and
the architecture need not be role-aligned.

\RACER{} keeps the useful part of both. Like L2D-Pop (QC), it is
query-conditioned. Like IFD, it removes absolute class-identity channels.
Unlike both, it exposes an explicit estimate of
$\Gamma^\star(x,y,C_e)=\Prob(M_E=y\mid X=x,Y=y,C_E=C_e)$, so calibration, budget ranking, and
multi-expert comparison can be based on a probability-scale estimate of
expert correctness rather than on a latent routing score.

\subsection{Multi-expert routing}
\label{ssec:multi-expert-routing}

If multiple experts are available with contexts $\{C_{e_j}\}_{j=1}^J$,
compute
\[
\widehat q_{\omega,j}(x)
\coloneqq 
\widehat q_\omega(x,C_{e_j})
=
\sum\nolimits_y p_\omega(y\mid x)\Gamma(x,y,C_{e_j})
\]
for each expert. With optional workload or cost penalty
$\lambda_j(x)\ge0$, choose
\[
j^\star(x)\coloneqq \argmax_j\{\widehat q_{\omega,j}(x)-\lambda_j(x)\}
\]
and defer only if
\[
\widehat q_{\omega,j^\star}(x)-\lambda_{j^\star}(x)-p_{\omega,\max}(x)
\ge\tau.
\]
Calibration is essential in this extension because the quantities
$\widehat q_{\omega,j}(x)$ must be comparable across experts and
contexts.

\section{Experiments}
\label{sec:experiments}

We evaluate whether a role-relative competence interface improves the two
quantities needed for population-adaptive deferral: budget-swept routing utility
and probability-scale expert-correctness estimation. The main experiments are
organized around three research questions:
\begin{enumerate}
    \item \emph{Can a method convert additional context for an unseen expert into better query-specific routing?}
    \item \emph{Does the method remain useful when the label space is large and expert competence varies across roles?}
    \item \emph{Do the resulting competence estimates behave like probabilities rather than only ranking scores?}
\end{enumerate}
A controlled context-posterior stress test for the optional \RACERC{} extension
is reported in \Cref{app:context-posterior-stress}. Full split details,
subtype construction, expert-generation parameters, and run counts are given in
Appendix~\ref{app:synthetic-protocol}.

\subsection{Evaluation setup}
\label{ssec:experimental-setup}

The synthetic benchmark has two main parts. The PathMNIST
\citep{pathmnist,medmnistv1,medmnistv2} context-scaling study fixes a locally
complementary expert setting and varies the context size, isolating whether a
router can exploit additional evidence when competence varies within class.
CIFAR-100 \citep{krizhevsky2009learning} stresses class cardinality and
synthetic expert variation by varying $K$, expert strength, hidden subtype dependence,
and expert-profile permutations.

The context-posterior
stress test is kept in the appendix because it is diagnostic rather than part
of the main method comparison: it asks when the more flexible \RACERC{}
extension is worth estimating.

Throughout the synthetic experiments, $K$ denotes the number of observable
classes and $B$ denotes the number of context examples supplied for the queried
expert. Expert competence can depend on a hidden within-class subtype. The
parameter $\rho$ controls this dependence: $\rho=0$ gives classwise competence,
where all subtypes of a class share the same accuracy, while $\rho=1$ gives full
hidden-subtype specialization. The nominal OOD permutation parameter
$\lambda_{\mathrm{id}}$ controls an expert-only permutation of the
class/subtype competence profile while image labels and classifier coordinates
remain fixed. The generator samples exchangeable class effects and independent
random subtype orderings for each class. Therefore a full permutation at
$\lambda_{\mathrm{id}}=1$ preserves the expert-profile population law, just as
$\lambda_{\mathrm{id}}=0$ does. We retain ``OOD'' as the recorded split label
for comparability with the saved results; these synthetic comparisons establish
performance on unseen simulated experts, not robustness to an established
expert-population distribution shift.

The
normalized context support $B/K$ is therefore the expected number of context
examples per observed role under a balanced context draw.

The real-expert benchmarks are VinDr-CXR \citep{PhysioNet-vindr-cxr-1.0.0} and
CheXpert \citep{irvin2019chexpert}, each of
which contains persistent multi-rater annotations for medical images. These
datasets are less controlled than the synthetic benchmark, but they test whether
the same role-relative interface remains useful when expert variability, label
noise, and context support are governed by real annotation processes rather than
by a simulator.

We follow \citet{strong2026identity} and use AURSAC, the area under the
system-accuracy curve obtained by sweeping the deferral threshold, as the primary
routing metric for multiclass tasks. Higher AURSAC means better routing over the
full range of deferral budgets. For multi-label chest radiography, we report
macro AURSBAC, the macro-average of per-pathology area under the system balanced
accuracy curve. For methods that expose a bounded expert-correctness probability
$\widehat q_\omega(x,C_e)$, we also report Brier score \cite{glenn1950verification} and 15-bin ECE \cite{naeini2015obtaining, guo2017calibrationmodernneuralnetworks} against
the realized event $\1{M_E=Y}$. L2D-Pop and classifier-confidence routers are
included in routing comparisons, but are omitted from calibration comparisons
because they do not directly output a finite expert-correctness probability.

RACER budget sweeps rank queries by estimated expert correctness minus
model confidence; the auxiliary deferral-head logit is used during training.
Appendix~\ref{app:implementation-evaluation} gives the method-specific losses,
weights, inference scores, and component updates, together with context sources,
context/query separation, and expert-selection protocols
(Tables~\ref{tab:implementation-protocol} and~\ref{tab:data-protocol}).

\subsection{Synthetic context scaling on PathMNIST}
\label{ssec:pathmnist-context-scaling}

The PathMNIST context-scaling experiment is the most direct test of the RACER
mechanism. We fix $K=9$, full subtype dependence ($\rho=1$), full expert-profile
permutation ($\lambda_{\mathrm{id}}=1$), and good/mid/bad subtype
accuracies $(0.98,0.70,0.30)$. The context size varies over
$B\in\{9,25,50,100,200,500,1000\}$, shown as $B/K$. At each context size, the
same context/query episodes are reused across methods and expert groups.
Each method is trained separately for each $B$; further sampling details are
in Appendix~\ref{app:pathmnist-context-protocol}.

\Cref{fig:pathmnist_context_gain} reports AURSAC gain over the
classifier-confidence router. At $B/K=1$, all adaptive methods are below the
classifier baseline, consistent with the fact that same-role context is too
sparse for reliable local competence estimation. As context grows, the methods
separate. \RACERKernel{} crosses above the classifier baseline around $B/K=5.6$
and shows increasing overall gain at the subsequent reported context sizes.
At $B/K=111$, it reaches an overall gain of $+0.0273$ AURSAC, with
$+0.0260$ on nominal unseen-OOD experts and $+0.0280$ on unseen-ID experts. \RACERKNN{}
also improves at large context, but reaches roughly half the final gain of the
learned kernel head. IFD-score and IFD-MLP improve mildly as their classwise
profiles become less noisy, while L2D-Pop QI/QC remain below the classifier
baseline throughout the sweep.

\begin{figure}[H]
\centering
\includegraphics[width=0.98\linewidth]{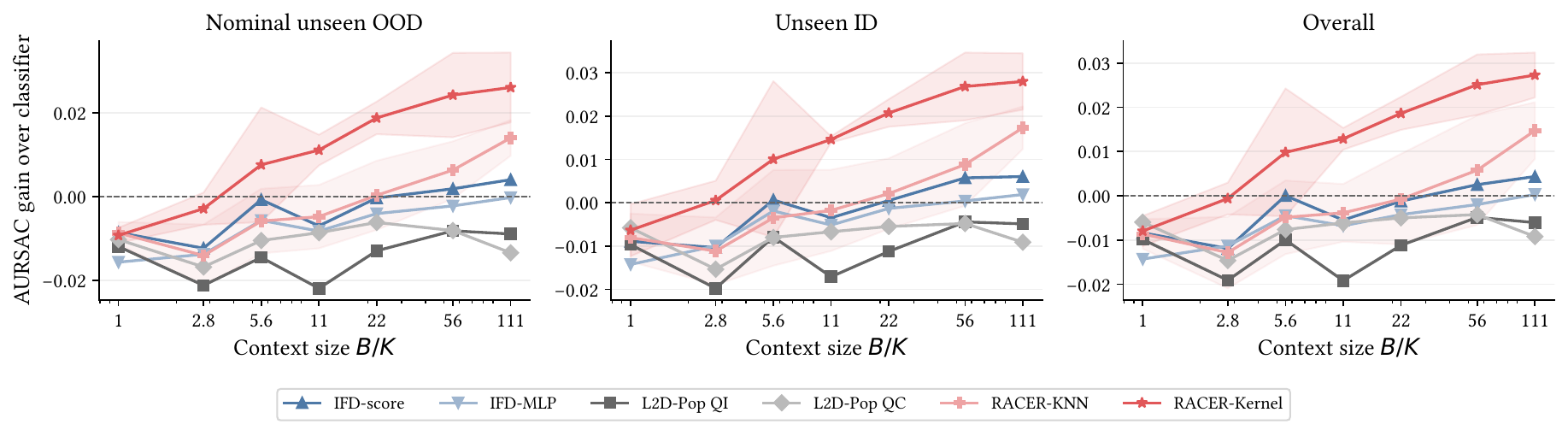}
\caption{PathMNIST context-scaling utility in the hardest locally complementary
synthetic setting ($\rho=1$, $\lambda_{\mathrm{id}}=1$). The y-axis shows AURSAC
gain over the classifier-confidence router; positive values indicate better
budget-swept routing than classifier uncertainty alone. The x-axis is normalized
context support $B/K$ with $K=9$. Shaded bands for RACER methods show one standard deviation over three seeds. We partition the results into (i) unseen OOD experts, (ii) unseen ID experts, and (iii) overall experts.}
\label{fig:pathmnist_context_gain}
\end{figure}

\Cref{fig:pathmnist_context_calibration} shows the corresponding calibration
curves. The Brier score follows the utility pattern: \RACERKernel{} is best
across the context range and improves from $0.2251$ at $B/K=1$ to $0.2029$ at
$B/K=111$. IFD-score and IFD-MLP approach Brier scores around $0.222$, while
\RACERKNN{} starts poorly at small context before converging to a similar range.
ECE is more nuanced. IFD-score and IFD-MLP have the lowest final ECE, while
\RACERKernel{} has slightly higher final ECE but substantially better Brier and
routing utility. In these results, smoother classwise
probabilities and query-conditioned estimates behave differently: a classwise profile can
be well bin-calibrated while remaining insufficiently discriminative for routing
when competence varies within class.

\begin{figure}[H]
\centering
\includegraphics[width=0.98\linewidth]{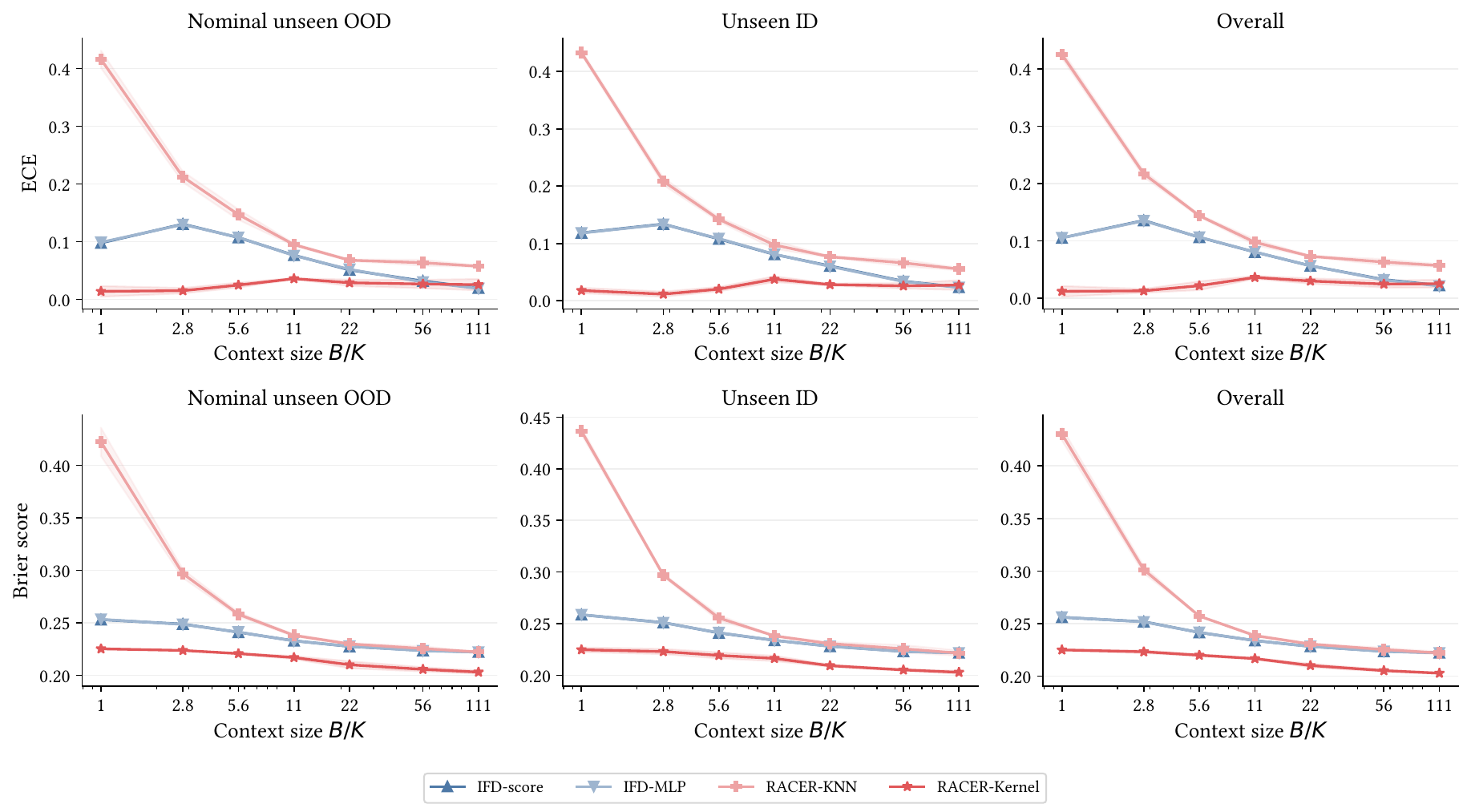}
\caption{PathMNIST context-scaling calibration. Mean ECE and Brier score are
reported for the plotted probability-output methods. RACER bands show one
standard deviation over three seeds. Lower is better.
The learned \RACERKernel{} head has the best Brier score among these plotted methods,
while classwise IFD variants have slightly lower final ECE but weaker routing
utility.}
\label{fig:pathmnist_context_calibration}
\end{figure}

This experiment explains why context size affects the interfaces differently.
Classwise profiles reduce variance as $B$ grows, but cannot remove within-class
bias. Latent population encoders observe the larger context, but in this stress
test do not convert it into a better routing margin. \RACERKernel{} uses the
additional evidence at the right level of structure: role-matched and query-near
context is smoothed by a role-shared competence head, yielding both lower Brier
score and larger routing gains. A representation-geometry ablation in
\Cref{app:representation-geometry-ablation} checks that this gain is not an
artifact of defining hidden subtypes in the same geometry used for local pooling:
\RACERKernel{} remains strongest when subtypes are generated from
classifier/auxiliary feature mixtures, while a random within-class subtype
negative control drives same-role neighbor purity to chance and removes the
\RACER{} gain.

To separate Racer's role-relative competence interface from the specific kernel-pooling estimator, we additionally evaluate \RACERDeepSets, a permutation-invariant neural competence estimator trained with the same BCE target and the same role-admissible features as RACER-kernel. \RACERDeepSets is a strong non-kernel baseline, but RACER-kernel remains best on PathMNIST context scaling, with monotonic context gains, higher final overall AURSAC gain at B=1000 (+0.0273 vs. +0.0194), and lower final Brier score (0.2029 vs. 0.2197). Full results are in App.  \ref{alb:deepsets}.

\subsection{Class-scale synthetic routing on CIFAR-100}
\label{ssec:synthetic-routing-results}

The CIFAR-100 experiments complement the PathMNIST sweep by stressing label-space
size and superclass-structured confusions. \Cref{tab:cifar_context} reports the
reference model, expert, and oracle accuracies. In the $K=20$ setting, the
classifier is already much stronger than the average expert, so the possible
improvement over classifier-confidence routing is modest. The $K=100$ setting is
more diagnostic: model and expert accuracy are closer, hidden subtype
specialization matters more, and the oracle gap leaves more room for better
routing.

\begin{table}[H]
\centering
\small
\setlength{\tabcolsep}{5pt}
\renewcommand{\arraystretch}{1.05}
\begin{tabular}{llccc}
\toprule
$K$ & Profile & Model acc. & Expert acc. & Oracle \\
\midrule
20 & Weak & $0.8609\pm0.0085$ & $0.6465\pm0.0113$ & $0.9005\pm0.0049$ \\
20 & Strong & $0.8629\pm0.0124$ & $0.6696\pm0.0164$ & $0.9095\pm0.0078$ \\
100 & Weak & $0.6560\pm0.0043$ & $0.6438\pm0.0072$ & $0.8255\pm0.0025$ \\
100 & Strong & $0.6548\pm0.0067$ & $0.6709\pm0.0171$ & $0.8362\pm0.0070$ \\
\bottomrule
\end{tabular}
\caption{OOD reference quantities for the synthetic CIFAR-100 experiments. The
oracle is an upper reference for routing and is not a learned method.
Variability is the standard deviation over four settings and three seeds.}
\label{tab:cifar_context}
\end{table}

\begin{figure}[H]
\centering
\includegraphics[width=0.95\linewidth]{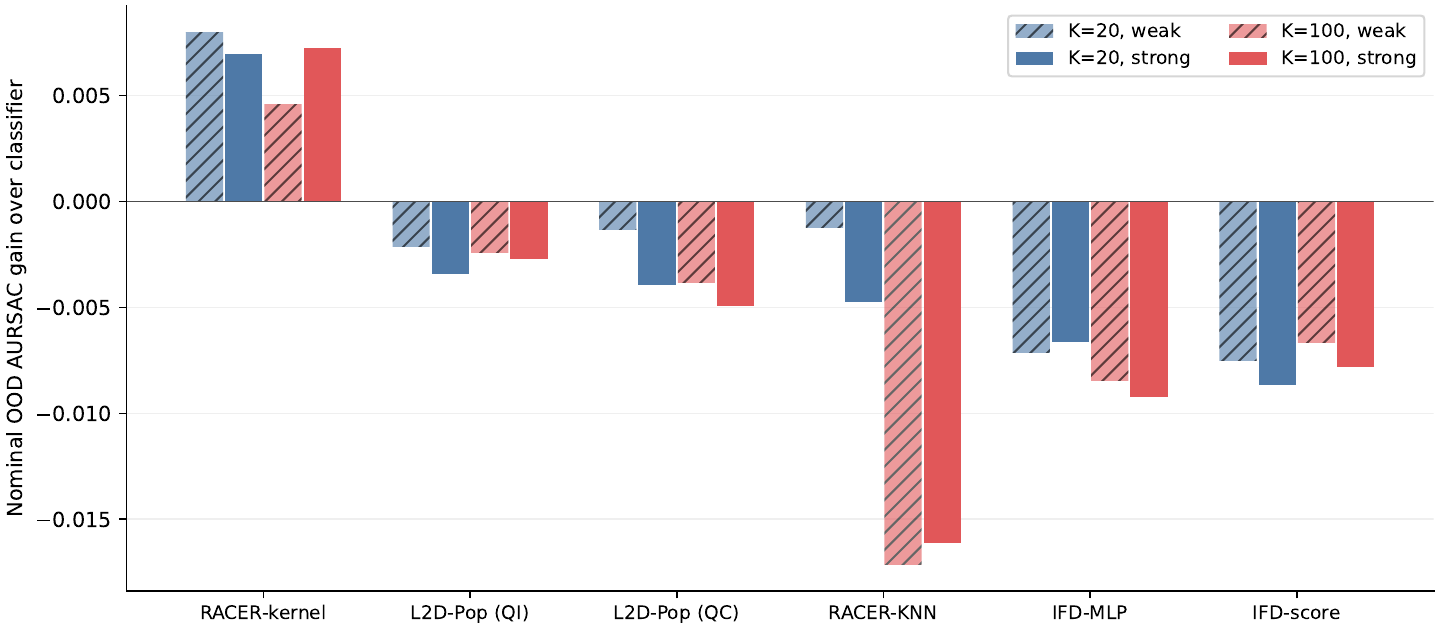}
\caption{OOD AURSAC gain over the classifier-confidence router. Across both
label-space sizes and both expert-strength profiles, \RACERKernel{} is the only
method with a positive aggregate OOD gain over the classifier-confidence
baseline. Bars show means over four settings and three seeds, without
inferential error bars. ``OOD'' denotes the recorded synthetic split.}
\label{fig:cifar_synthetic_ood_gain}
\end{figure}

\begin{table}[H]
\centering
\small
\setlength{\tabcolsep}{4pt}
\renewcommand{\arraystretch}{1.04}
\resizebox{\textwidth}{!}{%
\begin{tabular}{llccccc}
\toprule
$K$ & Method & Weak OOD & Strong OOD & $\Delta_{\mathrm{OOD}}$ & Weak overall & Strong overall \\
\midrule
\rowcolor{gray!8}\multirow{7}{*}{20}
& \RACERKernel{} \textit{(Ours)} & $\mathbf{0.8087}\pm0.0086$ & $\mathbf{0.8199}\pm0.0110$ & $+0.0112$ & $\mathbf{0.8072}\pm0.0032$ & $\mathbf{0.8222}\pm0.0105$ \\
& Classifier-conf. \textit{(Baseline)} & $0.8007\pm0.0063$ & $0.8130\pm0.0094$ & $+0.0122$ & $0.7996\pm0.0025$ & $0.8151\pm0.0085$ \\
& L2D-Pop (QI) \cite{tailor2024population} & $0.7986\pm0.0071$ & $0.8096\pm0.0091$ & $+0.0109$ & $0.7973\pm0.0021$ & $0.8124\pm0.0089$ \\
& L2D-Pop (QC) \cite{tailor2024population} & $0.7994\pm0.0083$ & $0.8090\pm0.0105$ & $+0.0096$ & $0.7983\pm0.0032$ & $0.8114\pm0.0094$ \\
& \RACERKNN{} \textit{(Ours)}& $0.7995\pm0.0075$ & $0.8082\pm0.0133$ & $+0.0087$ & $0.7969\pm0.0032$ & $0.8104\pm0.0132$ \\
& IFD-MLP \cite{strong2026identity} & $0.7936\pm0.0093$ & $0.8063\pm0.0099$ & $+0.0127$ & $0.7922\pm0.0055$ & $0.8078\pm0.0090$ \\
& IFD-score \cite{strong2026identity} & $0.7932\pm0.0078$ & $0.8043\pm0.0126$ & $+0.0111$ & $0.7907\pm0.0051$ & $0.8068\pm0.0120$ \\
\midrule
\rowcolor{gray!8}\multirow{7}{*}{100}
& \RACERKernel{} \textit{(Ours)} & $\mathbf{0.7299}\pm0.0050$ & $\mathbf{0.7461}\pm0.0091$ & $+0.0162$ & $\mathbf{0.7277}\pm0.0041$ & $\mathbf{0.7461}\pm0.0068$ \\
& Classifier-conf. \textit{(Baseline)}& $0.7253\pm0.0031$ & $0.7388\pm0.0084$ & $+0.0135$ & $0.7233\pm0.0034$ & $0.7391\pm0.0060$ \\
& L2D-Pop (QI) \cite{tailor2024population} & $0.7229\pm0.0043$ & $0.7361\pm0.0089$ & $+0.0132$ & $0.7207\pm0.0036$ & $0.7364\pm0.0071$ \\
& L2D-Pop (QC) \cite{tailor2024population} & $0.7215\pm0.0041$ & $0.7339\pm0.0099$ & $+0.0124$ & $0.7194\pm0.0045$ & $0.7342\pm0.0075$ \\
& IFD-score \cite{strong2026identity} & $0.7186\pm0.0040$ & $0.7310\pm0.0095$ & $+0.0124$ & $0.7165\pm0.0038$ & $0.7312\pm0.0083$ \\
& IFD-MLP \cite{strong2026identity} & $0.7169\pm0.0036$ & $0.7296\pm0.0097$ & $+0.0127$ & $0.7147\pm0.0030$ & $0.7297\pm0.0072$ \\
& \RACERKNN{} \textit{(Ours)}& $0.7082\pm0.0037$ & $0.7227\pm0.0093$ & $+0.0146$ & $0.7060\pm0.0032$ & $0.7227\pm0.0082$ \\
\bottomrule
\end{tabular}%
}
\caption{Synthetic CIFAR-100 routing utility. Entries are mean $\pm$ standard
deviation AURSAC over $4$ stress-test settings and $3$ seeds. The
``Classifier-conf.'' baseline sweeps a classifier-confidence deferral rule and
uses no expert context. $\Delta_{\mathrm{OOD}}$ is strong-minus-weak OOD
AURSAC.}
\label{tab:cifar_synthetic_utility}
\end{table}

\Cref{tab:cifar_synthetic_paired_gain} reports descriptive AURSAC differences
between \RACERKernel{} and classifier-confidence routing, matched by
$(\rho,\lambda_{\mathrm{id}},\mathrm{seed})$ within each $K$/profile block.
\RACERKernel{} is higher in 11/12, 11/12, 11/12, and 12/12 matched cells.
These cells share seeds and experimental factors; they are not asserted to be
independent replicates. We report no inferential tests or confidence intervals
for these comparisons in this checkpoint.

\begin{table}[H]
\centering
\small
\setlength{\tabcolsep}{6pt}
\renewcommand{\arraystretch}{1.04}
\begin{tabular}{llrc}
\toprule
$K$ & Expert profile & Cells & Mean $\Delta_{\mathrm{OOD}}$ \\
\midrule
20 & Weak & 12 & $+0.0080$ \\
20 & Strong & 12 & $+0.0070$ \\
100 & Weak & 12 & $+0.0046$ \\
100 & Strong & 12 & $+0.0072$ \\
\bottomrule
\end{tabular}
\caption{Descriptive paired CIFAR-100 nominal OOD AURSAC differences for \RACERKernel{} versus the
classifier-confidence router. $\Delta_{\mathrm{OOD}}$ is
AURSAC(\RACERKernel{}) minus AURSAC(classifier) on unseen-OOD experts.
Each block contains four settings and three seeds. The cell count describes
the experimental grid, not a count of independent replications.}
\label{tab:cifar_synthetic_paired_gain}
\end{table}

\Cref{fig:cifar_synthetic_ood_gain,tab:cifar_synthetic_utility,tab:cifar_synthetic_paired_gain}
support the
intended bias--variance role of the learned kernel head. \RACERKNN{} is
transparent and local, but its same-role neighborhoods become noisy when $K$ is
large and finite context support is sparse. \RACERKernel{} keeps the same
identity-free local evidence while learning how to smooth it using support,
similarity mass, posterior confidence, entropy, and global context accuracy. The
remaining oracle gap shows that finite-context instance-level competence
estimation is still the limiting factor, especially in the full $K=100$ label
space.

\subsection{CIFAR-100 expert-correctness calibration}
\label{ssec:synthetic-calibration-results}

Calibration on the CIFAR-100 stress tests asks whether the competence score can
be used as a probability, not merely as a ranking feature. This is central to
\RACER{} because the Bayes L2D decision compares model correctness and expert
correctness on a shared probability scale.

\begin{table}[H]
\centering
\footnotesize
\setlength{\tabcolsep}{5pt}
\renewcommand{\arraystretch}{1.1}

\begin{tabular}{llcccccc}
\toprule
\multirow{2}{*}{$K$} & \multirow{2}{*}{Profile}
& \multicolumn{2}{c}{IFD-score}
& \multicolumn{2}{c}{\RACERKNN{}}
& \multicolumn{2}{c}{\RACERKernel{}} \\
\cmidrule(lr){3-4}\cmidrule(lr){5-6}\cmidrule(lr){7-8}
& & Brier & ECE & Brier & ECE & Brier & ECE \\
\midrule
20  & Weak   & $0.2472$ & $0.1109$ & $0.2788$ & $0.1725$ & $\mathbf{0.2281}$ & $\mathbf{0.0209}$ \\
20  & Strong & $0.2397$ & $0.1106$ & $0.2727$ & $0.1753$ & $\mathbf{0.2206}$ & $\mathbf{0.0191}$ \\
100 & Weak   & $0.2504$ & $0.1126$ & $0.3274$ & $0.2500$ & $\mathbf{0.2293}$ & $\mathbf{0.0135}$ \\
100 & Strong & $0.2476$ & $0.1357$ & $0.3137$ & $0.2382$ & $\mathbf{0.2205}$ & $\mathbf{0.0133}$ \\
\bottomrule
\end{tabular}

\caption{OOD calibration of expert-correctness probabilities. Lower is better
for both Brier score and 15-bin ECE. Values are averaged over all
$(\rho,\lambda_{\mathrm{id}})$ settings and seeds.}
\label{tab:cifar_calibration}
\end{table}

\RACERKernel{} has the lowest Brier score and ECE in every $K$/profile setting,
with ECE between $0.013$ and $0.021$. Its mean predicted expert correctness is
also close to realized expert accuracy: for strong OOD experts,
$\widehat q_\omega=0.6712$ versus expert accuracy $0.6696$ at $K=20$, and
$\widehat q_\omega=0.6745$ versus $0.6709$ at $K=100$. IFD-score is
systematically under-confident in the high-cardinality strong-expert setting,
predicting $0.5352$ mean correctness against realized accuracy $0.6709$.
\RACERKNN{} is closer to unbiased on average but has poor bin-wise calibration,
which explains its high ECE and weak routing in \Cref{tab:cifar_synthetic_utility}.
Overall, local role-relative evidence is useful, but high-cardinality
finite-context deferral requires learned smoothing and calibration to turn that
evidence into a reliable probability-scale competence estimate.

\subsection{VinDr-CXR: real radiologist deferral}
\label{ssec:vindr-results}

\begin{figure}[t]
\centering
\includegraphics[width=\linewidth]{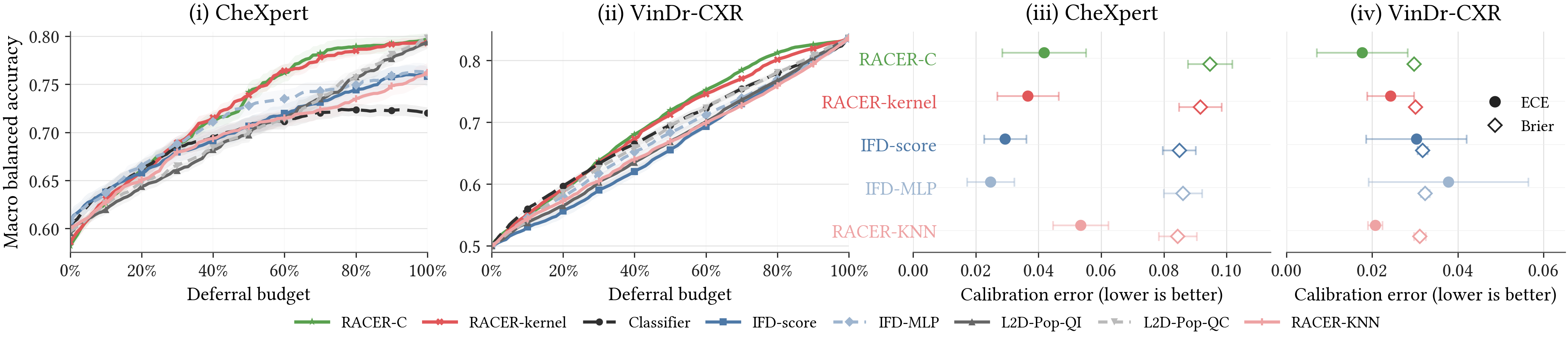}
\caption{Real-expert chest-radiography results. Panels (i)--(ii) show macro
balanced accuracy as the deferral budget increases; curves are means and shaded
bands denote standard errors over matched runs. Panels (iii)--(iv) summarize
expert-correctness calibration for methods with probability-scale outputs.}
\label{fig:chexpert_vindr_deferral_calibration}
\end{figure}

VinDr-CXR is a real-radiologist, multi-label deferral benchmark. We treat each
finding as a binary-relevance task and report macro AURSBAC. Because the
official VinDr-CXR test labels are consensus-only and have no radiologist
identity, all L2D evaluation uses the multi-radiologist training annotations:
for expert $e$, the deferred prediction is $e$'s own label and the target is the
leave-one-radiologist-out consensus of the other two readers. We use the
co-annotation-cluster holdout, which treats the dominant R8/R9/R10
cluster as the OOD expert group. No hospital shift is verified. Outcomes
measure agreement with the other two readers, with a pathology masked when
those readers disagree; they do not measure independent clinical ground truth. Protocol details, masking rules,
and the co-annotation counts are in \Cref{app:real-protocol-details}.

\Cref{tab:vindr-results} gives the five-seed summary. The pooled ``Overall''
metric sweeps one global deferral budget over all valid rows; the group columns
sweep the budget within each expert group. \RACERC{}
has the best pooled/global-budget AURSBAC, the lowest Brier score and ECE, and a
positive context-posterior NLL gain, indicating that its posterior adapter
learns useful case-mix information on average. However, \RACERKernel{} is
stronger in the group-conditional seen and unseen-OOD columns, and its
unreported group-weighted AURSBAC is higher ($0.7012$ versus $0.6950$ for
\RACERC{}). Thus \RACERC{} is useful when a single global budget may be
allocated unevenly across the pooled population, while \RACERKernel{} remains
the more robust default under group-conditional expert-shift evaluation.

The classifier-confidence baseline is also intentionally strong on VinDr-CXR.
This is not surprising: the target radiologist is very accurate in this
leave-one-radiologist-out construction, so simply deferring uncertain classifier
cases to that radiologist is already a competitive policy. In the five-seed
evaluation, the deferred expert has $97.00\pm0.03\%$ raw accuracy overall,
$98.62\pm0.41\%$ on unseen-ID cases, and $96.28\pm0.04\%$ on unseen-OOD cases,
whereas the image-only classifier is $90.26\pm0.04\%$ overall and
$87.80\pm0.06\%$ on unseen-OOD cases. The classifier baseline's strong OOD
performance should therefore be read as a property of the benchmark: when the
available expert is already highly reliable, the marginal room for learned
expert-specific routing is small. Detailed expert-accuracy diagnostics are in
\Cref{tab:vindr_expert_accuracy}.

\begin{table*}[t]
\centering
\scriptsize
\setlength{\tabcolsep}{1.6pt}
\renewcommand{\arraystretch}{1.08}
\begin{tabular*}{\textwidth}{@{\extracolsep{\fill}}lccccccc@{}}
\toprule
Method & Pooled & Seen & ID & OOD & Brier $\downarrow$ & ECE $\downarrow$ & Ctx. \\
\midrule
\RACERC{}
& \best{$0.7019\pm0.0016$}
& $0.7373\pm0.0180$
& $0.6359\pm0.0363$
& $0.6841\pm0.0040$
& \best{$0.0297\pm0.0004$}
& \best{$0.0177\pm0.0106$}
& \best{$0.0542\pm0.0250$} \\
\RACERKernel{}
& \second{$0.6973\pm0.0041$}
& \best{$0.7436\pm0.0155$}
& $0.6811\pm0.0413$
& \best{$0.6901\pm0.0057$}
& \second{$0.0301\pm0.0003$}
& $0.0243\pm0.0055$
& $0.0000\pm0.0000$ \\
Classifier-conf.
& $0.6878\pm0.0078$
& $0.7057\pm0.0152$
& \best{$0.7083\pm0.0466$}
& \second{$0.6898\pm0.0077$}
& -- & -- & $0.0000\pm0.0000$ \\
L2D-QC \citep{tailor2024population}
& $0.6839\pm0.0128$
& $0.7005\pm0.0605$
& $0.6806\pm0.0330$
& $0.6787\pm0.0110$
& -- & -- & $0.0000\pm0.0000$ \\
IFD-MLP \citep{strong2026identity}
& $0.6761\pm0.0098$
& $0.6749\pm0.0516$
& $0.6868\pm0.0629$
& $0.6832\pm0.0082$
& $0.0324\pm0.0010$
& $0.0378\pm0.0186$
& $0.0000\pm0.0000$ \\
L2D-QI \citep{tailor2024population}
& $0.6685\pm0.0174$
& $0.6788\pm0.0580$
& $0.6942\pm0.0478$
& $0.6805\pm0.0054$
& -- & -- & $0.0000\pm0.0000$ \\
\RACERKNN{}
& $0.6680\pm0.0028$
& $0.6779\pm0.0234$
& $0.6499\pm0.0191$
& $0.6768\pm0.0024$
& $0.0311\pm0.0014$
& \second{$0.0207\pm0.0017$}
& $0.0000\pm0.0000$ \\
IFD-score \citep{strong2026identity}
& $0.6606\pm0.0100$
& $0.6487\pm0.0853$
& $0.7055\pm0.0847$
& $0.6813\pm0.0059$
& $0.0318\pm0.0012$
& $0.0303\pm0.0117$
& $0.0000\pm0.0000$ \\
\bottomrule
\end{tabular*}
\caption{VinDr-CXR real-radiologist deferral. Overall AURSBAC is pooled over all
valid rows and therefore corresponds to one global deferral budget. Seen,
unseen-ID, and unseen-OOD columns compute group-conditional AURSBAC. Brier and
ECE evaluate expert-correctness probabilities and are omitted for methods that
only output routing scores. Ctx. gain is the label-NLL improvement of
$\pi_\rho(y\mid x,C_e)$ over the image-only posterior.}
\label{tab:vindr-results}
\end{table*}

\subsection{CheXpert: multi-expert human--AI deferral}
\label{ssec:chexpert-multiexpert-results}

CheXpert evaluates a different real-world regime: multiple deferral targets are
available at once, including human readers and external AI systems. The AI
experts are CARZero, KAD, and CheXZero. Each run chooses one candidate human
reader, includes two external AI systems as seen experts, and holds out the
third AI system until test time. The
primary metric is seen-plus-unseen macro AURSBAC. We use three candidate human
readers, three held-out AI choices, and three seeds, yielding 27 matched cells
per method. Details of the consensus construction, expert pools, and held-out
AI protocol are in \Cref{app:real-protocol-details}.

\Cref{tab:chexpert-multiexpert-results} reports the compact primary table. The
two learned \RACER{} variants are the strongest by seen-plus-unseen AURSBAC,
with nearly identical means. \RACERKernel{} is slightly stronger on seen experts
and remains our main method because it improves balanced-accuracy routing
without estimating the additional context-conditioned label posterior. \RACERC{}
has a similar pooled mean. Both learned variants have higher mean AURSBAC
than the IFD baselines in this table. These are descriptive comparisons;
this version makes no significance claim. The extended metrics are in
\Cref{app:real-extended-tables}.

\begin{table*}[t]
\centering
\small
\setlength{\tabcolsep}{4.5pt}
\renewcommand{\arraystretch}{1.08}
\begin{tabular*}{\textwidth}{@{\extracolsep{\fill}}lcccc@{}}
\toprule
Method & Seen+unseen AURSBAC & Seen AURSBAC & Brier $\downarrow$ & ECE $\downarrow$ \\
\midrule
\RACERC{} & \best{$0.7243\pm0.0341$} & $0.7244\pm0.0328$ & $0.0947\pm0.0071$ & $0.0418\pm0.0134$ \\
\RACERKernel{} & \second{$0.7241\pm0.0259$} & \best{$0.7276\pm0.0224$} & $0.0916\pm0.0068$ & $0.0366\pm0.0098$ \\
IFD-MLP \citep{strong2026identity} & $0.7105\pm0.0405$ & $0.7139\pm0.0349$ & $0.0861\pm0.0061$ & \best{$0.0247\pm0.0075$} \\
IFD-score \citep{strong2026identity} & $0.7013\pm0.0444$ & $0.7092\pm0.0431$ & $0.0849\pm0.0052$ & \second{$0.0293\pm0.0067$} \\
L2D-QC \citep{tailor2024population} & $0.7007\pm0.0219$ & $0.7026\pm0.0202$ & -- & -- \\
L2D-QI \citep{tailor2024population} & $0.6986\pm0.0265$ & $0.7029\pm0.0224$ & -- & -- \\
\RACERKNN{} & $0.6956\pm0.0360$ & $0.6987\pm0.0348$ & \best{$0.0844\pm0.0061$} & $0.0534\pm0.0088$ \\
Classifier-conf. & $0.6931\pm0.0360$ & $0.7001\pm0.0340$ & -- & -- \\
\bottomrule
\end{tabular*}
\caption{CheXpert multi-expert human--AI deferral. Values are mean $\pm$
standard deviation over 27 matched cells. Seen+unseen AURSBAC is the primary
metric; Brier and ECE are reported only for probability-output methods
and concern each method's selected expert, not common expert--query pairs.}
\label{tab:chexpert-multiexpert-results}
\end{table*}

\subsection{Experimental summary, limitations, and future directions}
\label{ssec:experimental-summary-limitations}

The experiments support three main conclusions. First, role-relative competence
estimation is most useful when expert skill has local structure that classifier
confidence alone cannot capture. This is clearest in the PathMNIST context sweep,
where \RACERKernel{} converts increasing
context into positive routing gain under hidden subtype dependence. Second,
identity-free local evidence must be smoothed carefully in high-cardinality
settings: \RACERKNN{} becomes noisy as same-role support shrinks, while
\RACERKernel{} remains the strongest aggregate router on the nominal OOD split of CIFAR-100 for both
$K=20$ and $K=100$. Third, probability-scale competence estimation is not the same
as maximizing a routing score. The learned kernel head gives the best Brier and
ECE on the CIFAR-100 calibration stress tests, whereas smoother classwise
profiles can have slightly lower ECE but weaker routing utility in the PathMNIST
large-context regime. The context-posterior stress test in
\Cref{app:context-posterior-stress} confirms this choice: when
$\lambda_{\mathrm{ctx}}=0$, \RACERC{} is essentially tied with
\RACERKernel{} (AURSAC $0.8157$ versus $0.8149$), while at
$\lambda_{\mathrm{ctx}}=1$, where context is maximally informative about case
mix, \RACERC{} improves AURSAC by $0.0131$ and reduces label NLL by $0.1154$.

Several limitations remain. The method relies on an informative query--context
geometry; if the image representation does not organize examples by the factors
that drive expert competence, local role-relative pooling cannot recover the
right competence law. Finite context is also a fundamental constraint,
especially when $K$ is large and the expected same-role support $B/K$ is small.
The real-expert benchmarks are valuable because they contain persistent raters,
but they cover a limited range of clinical tasks and annotation protocols, and
classifier confidence remains competitive. The higher reported mean for
\RACERKernel{} does not by itself establish a statistically reliable advantage;
this checkpoint reports descriptive comparisons without significance tests. Finally, calibration remains metric-dependent:
Brier score rewards sharp useful probabilities, while ECE can favor smoother but
less discriminative estimates.

\paragraph{Code availability.}
The code will be released upon publication.

\bibliographystyle{plainnat}
\bibliography{sample-base}
\newpage
\appendix
\renewcommand{\contentsname}{Contents of Appendix}
\tableofcontents
\addtocontents{toc}{\protect\setcounter{tocdepth}{3}}

\clearpage

\section{Detailed Proof Derivations}
\label{app:detailed-proofs}

This appendix expands the short proof sketches in the main text. The statements
and assumptions are the same as in the corresponding main-body propositions and
theorems; the goal here is only to show the intermediate algebra.

\subsection{Conditional minimizer of the augmented L2D surrogate}
\label{app:proof-augmented-softmax}

For completeness, we first derive the conditional minimizer used in
\Cref{ssec:l2d}. Fix $x$ and write $q(x)=\Prob(M=Y\mid X=x)$. The conditional
augmented-softmax risk is
\[
\calS(\Pi\mid x)
=
-\sum_{y\in\Y}\eta_y(x)\log \Pi_y
-q(x)\log \Pi_\perp,
\]
where $\Pi\in\Delta(\Yperp)$. Define nonnegative weights
\[
w_y\coloneqq \eta_y(x),\qquad w_\perp\coloneqq q(x),
\]
with total mass
\[
W\coloneqq \sum_y w_y+w_\perp=1+q(x).
\]
The optimization problem is
\[
\min_{\Pi_a\ge0,\,\sum_a\Pi_a=1}
-
\sum_{a\in\Yperp} w_a\log \Pi_a .
\]
For coordinates with $w_a>0$, the Lagrangian is
\[
\mathcal J(\Pi,\lambda)
\coloneqq  -\sum_a w_a\log\Pi_a + \lambda\left(\sum_a\Pi_a-1\right).
\]
The stationarity condition gives
\[
\frac{\partial \mathcal J}{\partial \Pi_a}
= -\frac{w_a}{\Pi_a}+\lambda=0,
\qquad\text{so}\qquad
\Pi_a=\frac{w_a}{\lambda}.
\]
Summing over $a$ and using $\sum_a\Pi_a=1$ yields
\[
1=\sum_a \frac{w_a}{\lambda}=\frac{W}{\lambda},
\qquad\text{hence}\qquad
\lambda=W=1+q(x).
\]
Therefore
\[
\Pi_y^\star(x)=\frac{\eta_y(x)}{1+q(x)},
\qquad
\Pi_\perp^\star(x)=\frac{q(x)}{1+q(x)}.
\]
Coordinates with zero weight receive zero probability in the closure of the
simplex; the same formula applies by continuity. The induced decision is
Bayes-aligned because
\[
\Pi_\perp^\star(x)\ge \max_y\Pi_y^\star(x)
\quad\Longleftrightarrow\quad
\frac{q(x)}{1+q(x)}\ge
\frac{\max_y\eta_y(x)}{1+q(x)}
\quad\Longleftrightarrow\quad
q(x)\ge \max_y\eta_y(x).
\]

\subsection{Expanded proof of \Cref{prop:proper-competence}}
\label{app:proof-proper-competence}

Let $A=\1{M_E=Y}$ and let $U=T(X,Y,C_E,p_\omega)$ be the information exposed to a
scalar predictor $g(U)\in(0,1)$. The binary cross-entropy risk is
\[
\calL(g)
\coloneqq 
\Ex\left[-A\log g(U)-(1-A)\log(1-g(U))\right].
\]
Using the tower property, this risk decomposes pointwise in $U$:
\[
\calL(g)
=
\Ex_U\left[
\Ex\left[-A\log g(U)-(1-A)\log(1-g(U))\mid U\right]
\right].
\]
Fix a value $U=u$ and define
\[
\alpha(u)\coloneqq \Prob(A=1\mid U=u).
\]
For a scalar prediction $s\in(0,1)$, the conditional risk is
\[
\ell_u(s)
\coloneqq  -\alpha(u)\log s -(1-\alpha(u))\log(1-s).
\]
Its first derivative is
\[
\ell_u'(s)
= -\frac{\alpha(u)}{s}+\frac{1-\alpha(u)}{1-s},
\]
and the stationarity condition $\ell_u'(s)=0$ gives
\[
\frac{\alpha(u)}{s}
=
\frac{1-\alpha(u)}{1-s}
\quad\Longleftrightarrow\quad
\alpha(u)(1-s)=(1-\alpha(u))s
\quad\Longleftrightarrow\quad
s=\alpha(u).
\]
The second derivative is
\[
\ell_u''(s)
=
\frac{\alpha(u)}{s^2}+\frac{1-\alpha(u)}{(1-s)^2}>0
\]
for $s\in(0,1)$ whenever $0<\alpha(u)<1$, so this stationary point is the unique
minimizer. If $\alpha(u)\in\{0,1\}$, the infimum is attained at the corresponding
boundary value by continuity. Hence the population minimizer over measurable
functions of $U$ is
\[
g^\star(U)=\Prob(A=1\mid U).
\]
When $U=(X,Y,C_E)$, this becomes
\[
g^\star(x,y,C_e)
=
\Prob(M_E=Y\mid X=x,Y=y,C_E=C_e).
\]
On the event $Y=y$, expert correctness $M_E=Y$ is the event $M_E=y$, so
\[
\Prob(M_E=Y\mid X=x,Y=y,C_E=C_e)
=
\Prob(M_E=y\mid X=x,Y=y,C_E=C_e)
=
\Gamma^\star(x,y,C_e).
\]
Thus binary cross-entropy is a strictly proper loss for the rolewise competence
probability, and a restricted summary $U=u_{e,y}(x)$ targets the corresponding
conditional projection $\Prob(A=1\mid u_{e,y}(x))$.

\subsection{Expanded proof of \Cref{prop:qhat-calibration}}
\label{app:proof-plugin-calibration}

By definition,
\[
q^\star(x,C_e)=\Prob(M_E=Y\mid X=x,C_E=C_e).
\]
Condition on the unknown true class and apply the law of total probability:
\[
q^\star(x,C_e)
=
\sum_{y\in\Y}
\Prob(Y=y\mid X=x,C_E=C_e)
\Prob(M_E=Y\mid X=x,Y=y,C_E=C_e).
\]
On the event $Y=y$, the event $M_E=Y$ is the same as $M_E=y$, so
\[
\Prob(M_E=Y\mid X=x,Y=y,C_E=C_e)
=
\Prob(M_E=y\mid X=x,Y=y,C_E=C_e)
=
\Gamma^\star(x,y,C_e).
\]
Under \Cref{assump:context-posterior-invariance},
\[
\Prob(Y=y\mid X=x,C_E=C_e)=\Prob(Y=y\mid X=x)=\eta_y(x).
\]
Therefore
\[
q^\star(x,C_e)
=
\sum_y \eta_y(x)\Gamma^\star(x,y,C_e).
\]
If $p_\omega(y\mid x)=\eta_y(x)$ and
$\Gamma(x,y,C_e)=\Gamma^\star(x,y,C_e)$ for all $y$, then the plug-in estimate
satisfies
\[
\widehat q_\omega(x,C_e)
=
\sum_y p_\omega(y\mid x)\Gamma(x,y,C_e)
=
\sum_y \eta_y(x)\Gamma^\star(x,y,C_e)
=
q^\star(x,C_e).
\]
This proves that the induced expert-correctness estimate is calibrated with
respect to the information $(x,C_e)$ when both the classifier posterior and the
rolewise competence functional equal their respective conditional probability targets.

\subsection{Expanded proof of \Cref{prop:plugin-regret}}
\label{app:proof-plugin-regret}

Fix $(X=x,C_E=C_e)$ and abbreviate
\[
q=q^\star(x,C_e),\qquad
\hat q=\widehat q_\omega(x,C_e),\qquad
\eta_{\max}=\max_y\eta_y(x),\qquad
\hat p_{\max}=p_{\omega,\max}(x).
\]
Let $y^\star\in\argmax_y\eta_y(x)$ and
$\hat y\coloneqq h_\omega(x)\in\argmax_y p_\omega(y\mid x)$. The Bayes conditional
correctness is
\[
\max\{\eta_{\max},q\}.
\]
The plug-in system predicts class $\hat y$ when
$\hat q<\hat p_{\max}$ and defers otherwise, so its conditional correctness is
\[
(1-\widehat r_0)\eta_{\hat y}+\widehat r_0 q.
\]
We first separate classifier error from routing error. If the plug-in routing
decision agrees with the Bayes comparison between $q$ and $\eta_{\max}$, then the
only possible loss relative to Bayes occurs when both systems predict, in which
case the plug-in classifier receives $\eta_{\hat y}$ instead of $\eta_{\max}$.
If the routing signs disagree, the additional loss is at most
$|q-\eta_{\max}|$. Hence the conditional excess risk is bounded by
\[
\eta_{\max}-\eta_{\hat y}
+
|q-\eta_{\max}|
\1{\operatorname{sign}(q-\eta_{\max})\neq
\operatorname{sign}(\hat q-\hat p_{\max})}.
\]
The classifier term is controlled by posterior error. Since
$p_\omega(\hat y\mid x)\ge p_\omega(y^\star\mid x)$,
\[
\begin{aligned}
\eta_{\max}-\eta_{\hat y}
&=(\eta_{y^\star}-p_{y^\star})+(p_{y^\star}-p_{\hat y})+(p_{\hat y}-\eta_{\hat y})\\
&\le |\eta_{y^\star}-p_{y^\star}|+|p_{\hat y}-\eta_{\hat y}|\\
&\le \|p_\omega(\cdot\mid x)-\eta(\cdot\mid x)\|_1.
\end{aligned}
\]
On the sign-disagreement event, the two real numbers
$q-\eta_{\max}$ and $\hat q-\hat p_{\max}$ have opposite signs or one is zero.
Therefore the distance from $q-\eta_{\max}$ to zero is no larger than its
distance to $\hat q-\hat p_{\max}$:
\[
|q-\eta_{\max}|
\le
|(q-\eta_{\max})-(\hat q-\hat p_{\max})|
\le
|q-
\hat q|+|\eta_{\max}-\hat p_{\max}|.
\]
The maximum posterior map is 1-Lipschitz in $\ell_\infty$, hence also in
$\ell_1$:
\[
|\eta_{\max}-\hat p_{\max}|
\le
\|p_\omega(\cdot\mid x)-\eta(\cdot\mid x)\|_\infty
\le
\|p_\omega(\cdot\mid x)-\eta(\cdot\mid x)\|_1.
\]
It remains to bound $|\hat q-q|$. Write
\[
\hat q-q
=
\sum_y p_y\Gamma_y-
\sum_y \eta_y\Gamma_y^\star
=
\sum_y(p_y-\eta_y)\Gamma_y
+
\sum_y\eta_y(\Gamma_y-\Gamma_y^\star),
\]
where $p_y\coloneqq p_\omega(y\mid x)$, $\Gamma_y\coloneqq \Gamma(x,y,C_e)$, and
$\Gamma_y^\star\coloneqq \Gamma^\star(x,y,C_e)$. Since $0\le\Gamma_y\le1$ and
$\sum_y\eta_y=1$,
\[
|\hat q-q|
\le
\sum_y |p_y-\eta_y|\,|\Gamma_y|
+
\sum_y\eta_y|\Gamma_y-\Gamma_y^\star|
\le
\|p_\omega(\cdot\mid x)-\eta(\cdot\mid x)\|_1
+
\sum_y\eta_y|\Gamma_y-\Gamma_y^\star|.
\]
Combining the three bounds gives the conditional excess risk
\[
\le
3\|p_\omega(\cdot\mid x)-\eta(\cdot\mid x)\|_1
+
\sum_y\eta_y(x)|\Gamma(x,y,C_e)-\Gamma^\star(x,y,C_e)|.
\]
Taking expectation over $(X,C_E)$ proves \Cref{prop:plugin-regret}.

\subsection{Expanded proof of \Cref{thm:cr-invariance}}
\label{app:proof-cr-invariance}

Fix a coherent class relabelling $\pi\in\mathfrak S_K$. The relabelled context is
\[
\pi C_e=\{(x_i^C,\pi(y_i^C),\pi(m_i^C))\}_{i=1}^B,
\]
and posterior equivariance means
\[
p_\omega^\pi(\pi(y)\mid x)=p_\omega(y\mid x)
\]
for every role $y$. We verify each ingredient of the RACER summaries.

First, same-role masks are preserved:
\[
\1{\pi(y_i^C)=\pi(y)}=\1{y_i^C=y}.
\]
Context correctness is also preserved:
\[
\1{\pi(m_i^C)=\pi(y_i^C)}=\1{m_i^C=y_i^C}.
\]
Image features and query--context similarities do not involve class names, so
$K_i(x)$ is unchanged. It follows that
\[
N_{e,\pi(y)}^\pi
=
\sum_i\1{\pi(y_i^C)=\pi(y)}
=
\sum_i\1{y_i^C=y}
=N_{e,y},
\]
and
\[
S_{e,\pi(y)}^\pi(x)
=
\sum_i\1{\pi(y_i^C)=\pi(y)}K_i(x)
=
\sum_i\1{y_i^C=y}K_i(x)
=S_{e,y}(x).
\]
The global context prior $\mu_{e,0}$ depends only on the preserved correctness
indicators and the context size, so it is unchanged. Therefore the smoothed local
statistic obeys
\[
\bar a_{e,\pi(y)}^\pi(x)
=
\frac{\alpha_0\mu_{e,0}+\sum_i\1{\pi(y_i^C)=\pi(y)}K_i(x)\1{\pi(m_i^C)=\pi(y_i^C)}}
{\alpha_0+S_{e,\pi(y)}^\pi(x)}
=
\bar a_{e,y}(x).
\]
Posterior equivariance preserves the role posterior value,
\[
p_\omega^\pi(\pi(y)\mid x)=p_\omega(y\mid x),
\]
the rank of the role,
\[
R_\omega^\pi(\pi(y);x)=R_\omega(y;x),
\]
the role margin, and the entropy of the posterior vector. Hence the full summary
satisfies
\[
u_{e,\pi(y)}^\pi(x)=u_{e,y}(x).
\]
For \RACERKNN{}, this immediately gives
\[
\Gamma_{\mathrm{KNN}}^\pi(x,\pi(y),\pi C_e)=\Gamma_{\mathrm{KNN}}(x,y,C_e).
\]
For \RACERKernel{}, the same equality holds because the MLP $g_\theta$ is shared
across roles:
\[
\Gamma_\theta^\pi(x,\pi(y),\pi C_e)
=
\sigm(g_\theta(u_{e,\pi(y)}^\pi(x)))
=
\sigm(g_\theta(u_{e,y}(x)))
=
\Gamma_\theta(x,y,C_e).
\]
Finally, change variables $y'=\pi(y)$ in the expert-correctness marginal:
\[
\begin{aligned}
\widehat q_\omega^\pi(x,\pi C_e)
&=
\sum_{y'\in\Y} p_\omega^\pi(y'\mid x)\,
\Gamma^\pi(x,y',\pi C_e) \\
&=
\sum_{y\in\Y} p_\omega^\pi(\pi(y)\mid x)\,
\Gamma^\pi(x,\pi(y),\pi C_e) \\
&=
\sum_{y\in\Y} p_\omega(y\mid x)\,
\Gamma(x,y,C_e) \\
&=
\widehat q_\omega(x,C_e).
\end{aligned}
\]
The maximum posterior is also invariant because relabelling only permutes the
coordinates:
\[
p_{\omega,\max}^\pi(x)=p_{\omega,\max}(x).
\]
Thus any plug-in threshold rule depending on
$\widehat q_\omega(x,C_e)-p_{\omega,\max}(x)$ is coherently invariant.

\subsection{Expanded proof of \Cref{thm:bayes-alignment}}
\label{app:proof-bayes-alignment}

Fix $(x,C_e)$ and consider the idealized conditional loss
\[
\ell(\Pi\mid x,C_e)
\coloneqq 
\Ex\left[-\log\Pi_Y-\Gamma^\star(x,Y,C_e)\log\Pi_\perp
\mid X=x,C_E=C_e\right].
\]
Expanding the expectation over $Y$ gives
\[
\ell(\Pi\mid x,C_e)
=
-
\sum_y\eta_y(x)\log\Pi_y
-
\sum_y\eta_y(x)\Gamma^\star(x,y,C_e)\log\Pi_\perp.
\]
By definition,
\[
q^\star(x,C_e)=\sum_y\eta_y(x)\Gamma^\star(x,y,C_e),
\]
so the conditional loss is
\[
\ell(\Pi\mid x,C_e)
=
-
\sum_y\eta_y(x)\log\Pi_y
-
q^\star(x,C_e)\log\Pi_\perp.
\]
This is a weighted cross-entropy over the augmented action space. The weights are
\[
w_y=\eta_y(x),\qquad w_\perp=q^\star(x,C_e),
\]
with total mass
\[
W=\sum_y\eta_y(x)+q^\star(x,C_e)=1+q^\star(x,C_e).
\]
The same Lagrange multiplier calculation as in
\Cref{app:proof-augmented-softmax} gives the minimizer
\[
\Pi_y^\star(x,C_e)=\frac{\eta_y(x)}{1+q^\star(x,C_e)},
\qquad
\Pi_\perp^\star(x,C_e)=\frac{q^\star(x,C_e)}{1+q^\star(x,C_e)}.
\]
The denominator is positive, so the induced defer/predict comparison is
\[
\begin{aligned}
\Pi_\perp^\star(x,C_e)\ge\max_y\Pi_y^\star(x,C_e)
&\quad\Longleftrightarrow\quad
\frac{q^\star(x,C_e)}{1+q^\star(x,C_e)}\ge
\frac{\max_y\eta_y(x)}{1+q^\star(x,C_e)}\\
&\quad\Longleftrightarrow\quad q^\star(x,C_e)\ge\max_y\eta_y(x).
\end{aligned}
\]
This is exactly the Bayes L2D comparison with the posterior-predictive expert
correctness probability conditioned on the available expert context.

\section{Relaxing Context-Posterior Invariance}
\label{app:context-posterior-extension}

The main method assumes \Cref{assump:context-posterior-invariance}, so that the
expert context is used to infer expert competence but not the label posterior of
the query. This appendix details \RACERC{}, the context-conditioned
classification extension of \RACERKernel{}, for deployments in which
$C_e$ may also reveal case mix, site, workflow, or acquisition information that
changes $\Prob(Y\mid X)$.

Without \Cref{assump:context-posterior-invariance}, the Bayes-relevant expert
correctness probability is
\begin{equation}
\label{eq:context-posterior-general-q}
q^\star(x,C_e)
=
\sum_{y\in\Y}
\underbrace{\Prob(Y=y\mid X=x,C_E=C_e)}_{\eta^C_y(x,C_e)}
\underbrace{\Prob(M_E=y\mid X=x,Y=y,C_E=C_e)}_{\Gamma^\star(x,y,C_e)}.
\end{equation}
The competence term is unchanged: \RACER{} still estimates
$\Gamma^\star(x,y,C_e)$. The additional object is the
context-conditioned label posterior
\[
\eta^C_y(x,C_e)\coloneqq \Prob(Y=y\mid X=x,C_E=C_e),
\]
which can be approximated by a role-equivariant posterior adapter
$\pi_\rho(y\mid x,C_e)$.

\paragraph{A role-equivariant posterior adapter.}
Let $p_\omega(y\mid x)$ be the image-only classifier posterior. A minimal
implementation keeps this posterior as a strong prior and learns a shared
role-relative logit adjustment,
\begin{equation}
\label{eq:racer-c-adapter}
b_{\rho,y}(x,C_e)
\coloneqq 
\log p_\omega(y\mid x)+a_\rho(x,y,C_e),
\qquad
\pi_\rho(y\mid x,C_e)\coloneqq \softmax_y b_{\rho,y}(x,C_e).
\end{equation}
The code implements the equivalent construction
$b_{\rho,y}=f_{\omega,y}+a_\rho$ using the original pre-softmax class logits.
The two forms differ only by a common softmax-normalization constant; in
particular, $a_\rho=0$ recovers $p_\omega$. This describes the computation
used by the supplied implementation. Here $a_\rho$ is produced by the same
network for every candidate role. Its
inputs should be role-admissible summaries, for example the image-only posterior
value and rank at role $y$, the same-role support $N_{e,y}$, the same-role
similarity mass $S_{e,y}(x)$, local context label density, and symmetric context
summaries. To preserve coherent relabelling equivariance, the adapter should not
use learned class embeddings, untied class-specific heads, or fixed
class-coordinate channels.

It is useful to separate two context channels conceptually. The posterior
adapter should use
\[
C_e^Y\coloneqq \{(x_i^C,y_i^C)\}_{i=1}^B
\]
to estimate case-mix information, while the competence head uses
\[
C_e^M\coloneqq \{(x_i^C,y_i^C,m_i^C)\}_{i=1}^B
\]
to estimate expert behavior. The expert predictions $m_i^C$ are therefore part
of the competence signal, not the default label-posterior signal.

The resulting \RACERC{} expert-correctness estimate is
\begin{equation}
\label{eq:racer-c-qhat}
\widehat q_{\rho,\theta}(x,C_e)
\coloneqq 
\sum_{y\in\Y}\pi_\rho(y\mid x,C_e)\Gamma_\theta(x,y,C_e).
\end{equation}
The autonomous prediction and plug-in rejector should use the same
context-conditioned posterior,
\begin{equation}
\label{eq:racer-c-rejector}
h_\rho(x,C_e)\coloneqq \argmax_{y\in\Y}\pi_\rho(y\mid x,C_e),
\qquad
\widehat r^{C}_\tau(x,C_e)
\coloneqq 
\1{\widehat q_{\rho,\theta}(x,C_e)-\max_y\pi_\rho(y\mid x,C_e)\ge \tau}.
\end{equation}
Using $\pi_\rho$ in \eqref{eq:racer-c-qhat} but comparing against
$p_{\omega,\max}(x)$ would be internally inconsistent: if the context changes
the label posterior, it changes both the expert-correctness marginalization and
the model-correctness side of the Bayes comparison.

\paragraph{Training.}
The posterior adapter can be trained episodically with a standard multiclass
proper loss,
\begin{equation}
\label{eq:racer-c-posterior-loss}
\calL_{\mathrm{post}}(\rho)
\coloneqq 
-\Ex\log \pi_\rho(Y\mid X,C_E^Y),
\end{equation}
where the target query is held out from its context. The competence head retains
\Cref{eq:gamma-functional,prop:proper-competence}'s binary proper loss for
$\Gamma_\theta$. A plug-in implementation can therefore minimize
\[
\calL_{\mathrm{post}}+\lambda_{\mathrm{comp}}\calL_{\mathrm{comp}},
\]
with an optional Bayes-aligned deferral surrogate obtained by replacing the
class logits in \Cref{ssec:bayes-aligned-training} with the context-conditioned
logits $b_{\rho,y}(x,C_e)$ and replacing $\widehat q_\omega$ with
$\widehat q_{\rho,\theta}$.

\paragraph{Why this is not the default method.}
\RACERC{} must estimate both the rolewise competence law $\Gamma^\star$ and
the context-conditioned label posterior $\eta^C$. Errors in this additional
posterior affect both expert-correctness estimates and classifier confidence.
Estimating it can increase variance when contexts are small or $B/K$ is low.
It may also learn associations with
context composition that fail to transfer across deployments. Under
\Cref{assump:context-posterior-invariance}, the target satisfies
$\eta^C_y(x,C_e)=\eta_y(x)$, so context provides no additional information
about the query label given $x$. We therefore use \RACERKernel{} as the
default under this assumption and evaluate \RACERC{} as an optional
extension when context is informative about query case mix.

\section{Computational Complexity}
\label{app:computational-complexity}

For one query and one candidate expert, a direct implementation of the rolewise
sums in \eqref{eq:same-role-mass}--\eqref{eq:local-smoothed-statistic} would
cost $O(KB)$ after image features are available. In practice the context
embeddings, labels, correctness indicators, $N_{e,y}$, and $\mu_{e,0}$ are
cached for each expert. The query embedding is computed once, the $B$
query--context similarities cost $O(Bd)$ for feature dimension $d$, and the
rolewise numerator and denominator are obtained by one scatter-add over the $B$
context items, costing $O(B+K)$. Applying the shared calibration MLP to all roles
costs $O(Kc_g)$ for a small per-role network cost $c_g$, and the final plug-in
sum is $O(K)$. Thus the optimized per-query, per-expert cost is
$O(Bd+B+Kc_g)$ plus the classifier forward pass, rather than $O(KB)$ pooling.
Routing among $J$ candidate experts scales linearly in $J$ but is batchable over
experts and roles. For very large $K$ or $B$, the same form admits standard
engineering approximations, such as evaluating only high-posterior roles or
using approximate nearest-neighbor context retrieval.

\section{Synthetic Protocol Details}
\label{app:synthetic-protocol}

This appendix collects protocol details omitted from the main experiments
section. The synthetic experiments have four parts: a PathMNIST context-scaling
study that fixes the hardest locally complementary expert setting and varies
context size, a representation-geometry ablation that tests whether subtype
recoverability rather than simulator leakage explains the PathMNIST gains, a
CIFAR-100 class-scale study that varies class cardinality, expert strength,
hidden subtype dependence, and expert-profile permutation, and a context-posterior
stress test that controls the degree to which the expert context reveals query
case mix.

\subsection{PathMNIST context-scaling protocol}
\label{app:pathmnist-context-protocol}

The PathMNIST context-scaling experiment uses the MedMNIST PathMNIST dataset at
64px resolution with all nine classes. The run contains $107{,}180$ total images,
split into $89{,}996$ training, $10{,}004$ validation, and $7{,}180$ test images.
The context size $B$ is varied while the class count is fixed at $K=9$, and the
main figures plot the normalized support ratio $B/K$:

\begin{table}[H]
\centering
\small
\setlength{\tabcolsep}{6pt}
\renewcommand{\arraystretch}{1.05}
\begin{tabular}{rrrrrrrr}
\toprule
$B$ & 9 & 25 & 50 & 100 & 200 & 500 & 1000 \\
$B/K$ & 1.0 & 2.8 & 5.6 & 11 & 22 & 56 & 111 \\
\bottomrule
\end{tabular}
\caption{PathMNIST context sizes used in the context-scaling ablation.}
\label{tab:pathmnist_context_sizes}
\end{table}

We use the strongest locally complementary synthetic expert profile: full
within-class subtype dependence ($\rho=1$), full expert-profile permutation
($\lambda_{\mathrm{id}}=1$), and good/mid/bad subtype accuracies
$(0.98,0.70,0.30)$ with three hidden subtypes per class. Experts are simulated
by extracting image features, clustering training images into class-conditional
hidden subtypes, assigning each expert a class/subtype competence tensor, and
sampling expert predictions for every image. ID experts preserve the class-role
profile; nominal OOD experts receive an expert-only permuted competence profile across
observed classes. Each run uses 32 seen training experts, 8 unseen-ID validation
experts, 8 unseen-OOD validation experts, 32 seen evaluation experts, 8 unseen-ID
test experts, and 8 unseen-OOD test experts. Early stopping uses unseen-ID
validation experts. At evaluation time, fixed context/query episodes are sampled
and reused across expert groups and methods; context sets are class-balanced when
possible. Each point uses three seeds.

Each method is trained separately for each $B$; this is not an inference-only
context sweep of a fixed checkpoint. Larger context also leaves fewer query
images within each evaluation split, so the query set need not be identical
across context sizes.

\subsection{Representation-geometry ablation}
\label{app:representation-geometry-ablation}

The main PathMNIST study uses class-conditional image-feature clusters to create
hidden expert subtypes, and \RACER{} also uses query--context geometry to recover
local competence. We therefore test whether the gains depend on using the same
representation to generate and retrieve hidden subtypes. Holding the PathMNIST
setting fixed at $K=9$, $B=500$, $\rho=1$,
$\lambda_{\mathrm{id}}=1$, and subtype accuracies $(0.98,0.70,0.30)$,
we generate hidden subtypes from mixtures of classifier features and auxiliary
image-quality features,
\begin{equation}
\label{eq:rep-ablation-mixture}
\psi_\alpha(x)
\coloneqq 
\mathrm{normalize}\!\left(
\alpha\phi_{\mathrm{cls}}(x)
+\sqrt{1-\alpha^2}\,\psi_{\mathrm{aux}}(x)
\right),
\end{equation}
with $\alpha\in\{1.0,0.75,0.5,0.25,0.0\}$. Here
$\phi_{\mathrm{cls}}$ is the default classifier/deep representation and
$\psi_{\mathrm{aux}}$ contains low-level color, contrast, texture, sharpness,
and stain-quality summaries. For this ablation only, we also evaluate an
auxiliary-geometry \RACERKernel{} variant that keeps the same competence
interface but uses $\psi_{\mathrm{aux}}$ for same-role similarity. Finally, we
include a random-subtype negative control in which hidden subtypes are assigned
independently within class. This destroys query--context recoverability while
preserving the same class/subtype expert-competence tensor.

\begin{table}[H]
\centering
\scriptsize
\setlength{\tabcolsep}{4pt}
\renewcommand{\arraystretch}{1.05}
\resizebox{\textwidth}{!}{%
\begin{tabular}{lcccccc}
\toprule
Method & $\alpha=1.00$ & $\alpha=0.75$ & $\alpha=0.50$ & $\alpha=0.25$ & $\alpha=0.00$ & Random \\
\midrule
Classifier-conf. & $0.8134\pm0.0099$ & $0.8193\pm0.0032$ & $0.8137\pm0.0117$ & $0.8172\pm0.0035$ & $0.8167\pm0.0044$ & \best{$0.8231\pm0.0011$} \\
IFD-score \citep{strong2026identity} & $0.8295\pm0.0034$ & $0.8357\pm0.0017$ & $0.8345\pm0.0067$ & $0.8324\pm0.0028$ & $0.8342\pm0.0054$ & $0.8081\pm0.0024$ \\
IFD-MLP \citep{strong2026identity} & $0.8114\pm0.0048$ & $0.8179\pm0.0053$ & $0.8151\pm0.0074$ & $0.8156\pm0.0056$ & $0.8193\pm0.0091$ & $0.8089\pm0.0076$ \\
L2D-Pop (QC) \citep{tailor2024population} & $0.8124\pm0.0038$ & $0.8104\pm0.0201$ & $0.8132\pm0.0036$ & $0.8106\pm0.0230$ & $0.8173\pm0.0040$ & $0.8122\pm0.0021$ \\
\RACERKNN{} \textit{(Ours)} & $0.8313\pm0.0052$ & $0.8342\pm0.0064$ & $0.8453\pm0.0076$ & $0.8330\pm0.0019$ & $0.8397\pm0.0101$ & $0.8121\pm0.0009$ \\
\RACERKernel{} \textit{(Ours)} & \best{$0.8432\pm0.0030$} & \best{$0.8583\pm0.0045$} & \best{$0.8563\pm0.0052$} & \best{$0.8573\pm0.0025$} & \best{$0.8593\pm0.0068$} & $0.8136\pm0.0010$ \\
\RACERKernel{}-aux \textit{(Ours)} & $0.8312\pm0.0026$ & $0.8356\pm0.0027$ & $0.8367\pm0.0074$ & $0.8342\pm0.0122$ & $0.8363\pm0.0055$ & $0.8194\pm0.0009$ \\
\bottomrule
\end{tabular}%
}
\caption{PathMNIST representation-geometry ablation. Entries are overall AURSAC
mean $\pm$ standard deviation over three seeds at $B=500$, $K=9$,
$\rho=1$, and $\lambda_{\mathrm{id}}=1$. The $\alpha$ columns generate hidden
subtypes from classifier/auxiliary feature mixtures using
\eqref{eq:rep-ablation-mixture}. The random column assigns hidden subtypes
independently within class. Bold denotes the best method in each column.}
\label{tab:pathmnist_representation_ablation_aursac}
\end{table}

\begin{table}[H]
\centering
\small
\setlength{\tabcolsep}{6pt}
\renewcommand{\arraystretch}{1.05}
\begin{tabular}{lcccc}
\toprule
Subtype source & \RACERKernel{} gain & Same-role purity & Brier $\downarrow$ & ECE $\downarrow$ \\
\midrule
$\alpha=1.00$ & $+0.0298$ & $0.6013$ & $0.2017$ & $0.0234$ \\
$\alpha=0.75$ & $+0.0390$ & $0.7609$ & $0.1850$ & $0.0207$ \\
$\alpha=0.50$ & $+0.0426$ & $0.7888$ & $0.1800$ & $0.0184$ \\
$\alpha=0.25$ & $+0.0401$ & $0.7819$ & $0.1838$ & $0.0259$ \\
$\alpha=0.00$ & $+0.0425$ & $0.7853$ & $0.1787$ & $0.0142$ \\
Random & $-0.0095$ & $0.3338$ & $0.2247$ & $0.0122$ \\
\bottomrule
\end{tabular}
\caption{Mechanism diagnostics for \RACERKernel{} in the representation-geometry
ablation. Gain is AURSAC improvement over the classifier-confidence router.
Same-role purity is the fraction of the top-$k$ same-true-role context neighbors
that share the query's hidden subtype under the geometry used by
\RACERKernel{}. With three hidden subtypes, chance purity is $1/3$.}
\label{tab:pathmnist_representation_ablation_diagnostics}
\end{table}

\RACERKernel{} remains the strongest method across all structured
classifier/auxiliary subtype geometries. Its gain over the classifier increases
from $+0.0298$ AURSAC at $\alpha=1.0$ to roughly $+0.039$--$+0.043$ for
$\alpha<1.0$. The same-role purity diagnostic explains why this is not a
representation-mismatch failure case: auxiliary-defined subtypes are still highly
recoverable by the default \RACERKernel{} geometry, with purity around
$0.76$--$0.79$ for $\alpha<1.0$. In contrast, random within-class subtypes drive
purity to chance and remove the routing gain, with \RACERKernel{} falling
$0.0095$ AURSAC below the classifier. These results support the intended
mechanism: role-relative local competence estimation helps when expert-relevant
within-class structure is recoverable from context geometry, and it does not
spuriously improve when that structure is absent.

\subsection{Context-posterior stress-test protocol}
\label{app:context-posterior-stress}

The context-posterior stress test reuses the PathMNIST label space and synthetic
expert simulator, but adds a latent case-mix domain $S\in\{1,\ldots,L\}$. For
each domain $s$, we sample a domain-specific class prior $\alpha_s\in
\Delta^{K-1}$ and define a global prior $\alpha_0$. Context labels are drawn
from $\alpha_s$, while query labels are drawn from
\begin{equation}
\label{eq:ctx-stress-query-prior}
\alpha_{\lambda}(s)
\coloneqq 
(1-\lambda_{\mathrm{ctx}})\alpha_0
+\lambda_{\mathrm{ctx}}\alpha_s.
\end{equation}
Thus $\lambda_{\mathrm{ctx}}=0$ satisfies the intended context-posterior
invariance regime, while $\lambda_{\mathrm{ctx}}=1$ makes the context maximally
informative about query case mix. Expert competence generation is held fixed
across the sweep, so the experiment isolates context-posterior shift rather
than changing the expert-reliability task. \RACERC{} uses the posterior adapter
from \Cref{app:context-posterior-extension}; all other components are matched to
\RACERKernel{}.

\begin{table}[H]
\centering
\small
\setlength{\tabcolsep}{5pt}
\renewcommand{\arraystretch}{1.05}
\begin{tabular}{lccccc}
\toprule
$\lambda$ & Kernel AURSAC & \RACERC{} AURSAC & Kernel Brier & \RACERC{} Brier & $\Delta$ NLL \\
\midrule
0.00 & $0.8149\pm0.0034$ & $0.8157\pm0.0043$ & $0.2254\pm0.0014$ & $0.2253\pm0.0015$ & $0.0050\pm0.0056$ \\
0.25 & \best{$0.8173\pm0.0007$} & $0.8160\pm0.0033$ & $0.2252\pm0.0003$ & $0.2251\pm0.0007$ & $-0.0032\pm0.0170$ \\
0.50 & $0.8187\pm0.0058$ & \best{$0.8229\pm0.0039$} & \best{$0.2239\pm0.0005$} & $0.2244\pm0.0012$ & $-0.0212\pm0.0058$ \\
1.00 & $0.8211\pm0.0071$ & \best{$0.8342\pm0.0040$} & $0.2222\pm0.0018$ & \best{$0.2213\pm0.0008$} & $-0.1154\pm0.0018$ \\
\bottomrule
\end{tabular}
\caption{Context-posterior stress test. $\Delta$ label NLL is
$\mathrm{NLL}(\pi_\rho)-\mathrm{NLL}(p_\omega)$, so negative values indicate
that the context-conditioned posterior improves held-out label-posterior fit.
At small violation strengths \RACERC{} is tied with \RACERKernel{}; when the
context strongly carries case-mix information, \RACERC{} improves both label
fit and routing.}
\label{tab:context_posterior_stress}
\end{table}

The results support the modelling tradeoff. Under the intended regime
($\lambda_{\mathrm{ctx}}=0$), the posterior adapter provides essentially no
routing benefit. As the violation strengthens, the adapter becomes useful: at
$\lambda_{\mathrm{ctx}}=1$, \RACERC{} improves AURSAC by $0.0131$ and reduces
label NLL by $0.1154$. We therefore keep \RACERKernel{} as the default estimator
and use \RACERC{} only as an extension for deployments where the context is
allowed to reveal stable case-mix information.

\subsection{CIFAR-100 class-scale protocol}
\label{app:cifar-protocol}

The $K=100$ setting uses all CIFAR-100 fine labels. The $K=20$ setting
uses a seeded superclass-balanced subset: CIFAR-100 has twenty superclasses, and
we select exactly one fine class from each superclass using seed $0$. Thus
$K=20$ is a balanced low-cardinality stress test rather than the first twenty
CIFAR classes, while $K=100$ retains the full fine-label and
superclass-confusion structure.

\begin{table}[H]
\centering
\small
\setlength{\tabcolsep}{5pt}
\renewcommand{\arraystretch}{1.05}
\begin{tabular}{lrrrr}
\toprule
Setting & Fine classes & Train & Val. & Test \\
\midrule
$K=20$ & 20 & 9{,}000 & 1{,}000 & 2{,}000 \\
$K=100$ & 100 & 45{,}000 & 5{,}000 & 10{,}000 \\
\bottomrule
\end{tabular}
\caption{Synthetic CIFAR-100 label-space splits. The canonical training
split is stratified into train/validation subsets.}
\label{tab:cifar_splits}
\end{table}

\begin{table}[H]
\centering
\scriptsize
\setlength{\tabcolsep}{5pt}
\renewcommand{\arraystretch}{1.05}
\begin{tabular}{llll}
\toprule
Superclass & Selected fine class & Superclass & Selected fine class \\
\midrule
\texttt{aquatic\_mammals} & \texttt{beaver} &
\texttt{fish} & \texttt{shark} \\
\texttt{flowers} & \texttt{poppy} &
\texttt{food\_containers} & \texttt{bottle} \\
\texttt{fruit\_and\_vegetables} & \texttt{mushroom} &
\texttt{household\_electrical} & \texttt{clock} \\
\texttt{household\_furniture} & \texttt{chair} &
\texttt{insects} & \texttt{beetle} \\
\texttt{large\_carnivores} & \texttt{wolf} &
\texttt{large\_manmade\_outdoor} & \texttt{bridge} \\
\texttt{large\_natural\_outdoor} & \texttt{forest} &
\texttt{large\_omnivores\_herbivores} & \texttt{elephant} \\
\texttt{medium\_mammals} & \texttt{raccoon} &
\texttt{noninsect\_invertebrates} & \texttt{crab} \\
\texttt{people} & \texttt{boy} &
\texttt{reptiles} & \texttt{lizard} \\
\texttt{small\_mammals} & \texttt{mouse} &
\texttt{trees} & \texttt{oak\_tree} \\
\texttt{vehicles\_1} & \texttt{train} &
\texttt{vehicles\_2} & \texttt{streetcar} \\
\bottomrule
\end{tabular}
\caption{The seeded superclass-balanced $K=20$ subset. Exactly one fine
class is selected from each CIFAR-100 superclass.}
\label{tab:cifar_k20_subset}
\end{table}

For every image, we construct an unobserved subtype. We extract frozen ResNet-18
image features, cluster the training images within each true class into three
subtypes, and assign validation and test images to the nearest within-class
subtype centroid. Expert correctness then depends on the pair
$(Y,\text{subtype})$. When an expert is incorrect, the wrong label is sampled
from a structured confusion prior that places most of its mass on fine classes
from the same CIFAR-100 superclass when such alternatives are available.

We use two expert-strength profiles. The weak profile is the original setting
with base expert accuracy $0.65$, class-skill standard deviation $0.55$,
global-skill standard deviation $0.20$, and three subtypes. When within-class
specialization is active, its good/mid/bad subtype accuracies are approximately
$0.90$, the sampled classwise baseline around $0.65$, and $0.40$. The strong
profile uses a cleaner locally complementary structure: good, mid, and bad
subtypes have accuracies $0.98$, $0.70$, and $0.30$, respectively.

Two binary factors define the stress-test grid. The subtype factor
$\rho\in\{0,1\}$ controls whether competence is classwise or instance-dependent:
$\rho=0$ makes all subtypes within a class share the same skill, while $\rho=1$
activates full hidden-subtype specialization. The permutation factor
$\lambda_{\mathrm{id}}\in\{0,1\}$ controls whether the nominal OOD group
receives an expert-only permutation of its classwise competence profile.
Because class effects and subtype assignments are exchangeable across classes,
this full permutation preserves the population law. It changes a realized
expert's mapping to fixed roles, but does not create a population shift.
The synthetic ``OOD'' results therefore concern a separately sampled unseen
expert group, not verified out-of-distribution transfer.

Each run uses 32 seen-training experts, 8 unseen-ID validation experts, 8
unseen-OOD validation experts, 32 seen-evaluation experts, 8 unseen-ID test
experts, and 8 unseen-OOD test experts. Early stopping uses the unseen-ID
validation experts. For each $K$ and expert-strength profile, the grid contains
seven methods, four $(\rho,\lambda_{\mathrm{id}})$ settings, and three seeds,
for $84$ runs; all four grids completed $84/84$ runs.

\section{Extended Literature Review}
\label{app:extended-lit-review}
\subsection{Learning with rejection and learning to defer}
Learning to defer is closely related to selective classification and
learning with a rejection option. Classical rejection methods abstain
when the model is uncertain, often comparing model confidence to a fixed
cost of rejection. L2D replaces the fixed rejection cost with a human or
external expert whose correctness depends on the input. This changes the
Bayes rule from ``reject when the model is uncertain'' to ``defer when the
expert is more likely to be correct than the model.'' Early L2D
formulations introduced this human--AI risk objective, while later work
developed consistent surrogates and practical neural training objectives
\citep{madras2018predict,mozannar2020consistent}.
The distinction matters for our setting. A selective classifier only
needs to estimate its own reliability. A deferral system must compare
two agents: the model and the expert. For unseen experts, the expert
side of this comparison must be inferred from context. This is why our
method focuses on estimating expert competence rather than only learning
a rejector.
\subsection{Clinical collaboration and hierarchical deferral}
The scoping review of \citet{strong2026haic} examines human--AI
collaboration across healthcare tasks, highlighting the roles of workflow,
interaction design, and appropriate reliance alongside predictive
performance. A concrete application is the guided deferral system of
\citet{strong2025guided}, which classifies medical reports with large
language models and supplies guidance to humans on deferred cases.
\RACER{} focuses on estimating the correctness of an available expert
from context; it does not generate guidance for that expert.

\citet{strong2026hierarchical} study coherent deferral for hierarchical
multi-label medical imaging, constraining prediction and delegation actions
to respect the label taxonomy and handoff semantics. This concerns the
compatibility of actions across related labels. The coherent relabelling
studied here instead concerns invariance to class names in a multiclass
decision. Extending \RACER{}'s competence estimates to hierarchical
deferral would additionally require an appropriate structured decision rule.

\subsection{Calibration in fixed-expert and fixed multi-expert L2D}
Several recent papers study calibration in L2D.
\citet{verma2022calibrated} show that the standard augmented-softmax
surrogate of \citet{mozannar2020consistent} is decision-consistent but
does not directly yield calibrated estimates of expert correctness. The
relevant expert-correctness estimator is an odds transform of the
deferral coordinate and can exceed $1$. Their one-vs-all alternative
directly parameterizes expert correctness with a bounded sigmoid and
improves calibration.
\citet{verma2023learning} extend the calibration analysis to fixed
multi-expert L2D. They show that symmetric softmax formulations can
couple expert-correctness estimates across experts, so the estimate for
one expert depends on the deferral coordinates of other experts. This is
undesirable when one wants calibrated expert-specific probabilities.
\citet{cao2023defense} further show that the problem is not the use of
softmax itself, but the symmetric treatment of class probabilities and
expert correctness. The correct probability object has product
structure: class probabilities live in a simplex, while expert
correctness is a separate scalar.
Our work builds on this lesson but changes the problem setting. Prior
calibrated L2D work asks whether a fixed-expert or fixed multi-expert
system can estimate $\Prob(M=Y\mid X=x)$ or $\Prob(M_j=Y\mid X=x)$.
\RACER{} asks whether a system can estimate $\Prob(M_E=Y\mid X=x,C_E=C_e)$
for an unseen expert represented only by a small context. This
introduces an additional statistical challenge: the expert-correctness
probability must be inferred from a set-valued context and should
generalize across experts.
\subsection{Population-adaptive deferral and unseen experts}
Population-adaptive L2D studies deferral to experts not necessarily
observed during training. L2D-Pop learns to condition a rejector on a
small context set of expert behavior \citep{tailor2024population}.
Query-independent variants compress the context into a fixed expert
representation; query-conditioned variants allow the current input to
attend to relevant context examples. This makes them expressive enough
to model instance-dependent expertise.
However, high-capacity context encoders can exploit absolute class
identity when they consume one-hot labels, class embeddings, or
class-specific channels. This creates a failure mode under coherent
relabellings or shifted expert populations. IFD addresses this by
removing expert identity and absolute class-coordinate shortcuts, but it
does so through classwise competence profiles
\citep{strong2026identity}. \RACER{} keeps the identity-free symmetry
while restoring query dependence.
\subsection{Decision consistency, probability estimation, and representation}
\label{app:decision-vs-probability}
A subtle point in comparing L2D-Pop, IFD, and \RACER{} is that all three
can be Bayes-aligned as defer/predict decision rules, yet they differ
substantially as estimators of expert competence. This distinction is
central to our empirical findings.
\paragraph{The augmented softmax is a decision surrogate.}
For a fixed query $x$ and expert context $C_e$, let
\[
\eta_y(x)\ \text{denote the label posterior},
\qquad
q^\star(x,C_e)=\Prob(M_E=Y\mid X=x,C_E=C_e).
\]
The conditional population form of the augmented softmax L2D loss is
\[
\mathcal S(\Pi\mid x,C_e)
\coloneqq 
-\sum_y \eta_y(x)\log \Pi_y(x,C_e)
-
q^\star(x,C_e)\log \Pi_\perp(x,C_e),
\]
where $\Pi$ is a distribution on $\Y\cup\{\perp\}$. By a Lagrange
multiplier calculation, the unique minimizer over the simplex is
\[
\Pi_y^\star(x,C_e)
=
\frac{\eta_y(x)}{1+q^\star(x,C_e)},
\qquad
\Pi_\perp^\star(x,C_e)
=
\frac{q^\star(x,C_e)}{1+q^\star(x,C_e)}.
\]
Therefore
\[
\Pi_\perp^\star(x,C_e)\ge \max_y \Pi_y^\star(x,C_e)
\quad\Longleftrightarrow\quad
q^\star(x,C_e)\ge \max_y \eta_y(x).
\]
Equivalently, at the logit level,
\[
g_\perp^\star(x,C_e)-g_y^\star(x)
=
\log q^\star(x,C_e)-\log \eta_y(x),
\]
up to an additive constant shared by all logits. The augmented softmax
loss therefore learns the correct Bayes comparison. This is why
one-stage L2D training is powerful: the classifier and rejector can
co-adapt to optimize system performance rather than being trained as
isolated modules.
\paragraph{Decision consistency is not the same as calibrated competence
estimation.}
The previous calculation shows that the augmented softmax is a good
surrogate for the decision boundary. It does not imply that the deferral
coordinate is itself a calibrated estimate of expert correctness. At the
optimum,
\[
\Pi_\perp^\star(x,C_e)=\frac{q^\star(x,C_e)}{1+q^\star(x,C_e)},
\]
so recovering $q^\star$ requires the odds transform
\[
q^\star(x,C_e)=\frac{\Pi_\perp^\star(x,C_e)}{1-\Pi_\perp^\star(x,C_e)}.
\]
Away from the ideal optimum, this transform is not constrained to lie in
$[0,1]$. Thus the augmented action probability is best viewed as a
normalized decision variable, not as the expert-correctness probability
itself. This distinction is harmless if the goal is only asymptotic
defer/predict consistency, but it matters when the output is used for
calibration, budget ranking, multi-expert comparison, or transfer to
unseen experts.
\paragraph{L2D-Pop learns a latent routing margin.}
L2D-Pop extends the augmented softmax surrogate to population-adaptive
deferral by conditioning the deferral logit on a context representation:
\[
d_{\mathrm{Pop}}(x,C_e)
\coloneqq 
g_\perp(x,\psi(C_e))
\quad\text{or}\quad
d_{\mathrm{Pop}}(x,C_e)
\coloneqq 
g_\perp(x,\psi(x,C_e)).
\]
At the population optimum, the resulting margin can recover the Bayes
comparison. However, the method is not asked to estimate the components
of that comparison separately. The context encoder and rejector learn
the scalar quantity
\[
d_{\mathrm{Pop}}(x,C_e)-\max_y g_y(x)
\approx
\log q^\star(x,C_e)-\log p_{\max}(x).
\]
This scalar can fit the training routing objective even when its
internal explanation is not a transferable expert-competence law. For
example, a high deferral margin may be caused by high expert competence,
low model confidence, a frequent class, a site-specific artifact,
context-set composition, or a fixed class-coordinate shortcut. In ID
settings these explanations may be correlated and therefore empirically
useful. Under expert shift they can come apart.
This is especially important because standard population encoders often
encode class labels and expert predictions in fixed coordinates:
\[
[\varphi(x_i^C),e(y_i^C),e(m_i^C)].
\]
Such encodings allow the rejector to learn identity-conditioned rules,
e.g. ``defer when coordinate $j$ is strong.'' These rules are not ruled
out by the augmented softmax objective, because they can reduce training
risk whenever coordinate identity is predictive in the training
population. But they do not express the invariant rule needed for
transfer: ``defer when the expert is strong on the relevant role.''
\paragraph{IFD estimates competence explicitly, but classwise.}
IFD makes the opposite tradeoff. From context counts
\[
n_{e,y}=\sum_i \1{y_i^C=y},
\qquad
t_{e,y}=\sum_i \1{y_i^C=y,\ m_i^C=y_i^C},
\]
it constructs a Bayesian profile
\[
\theta_{e,y}\mid C_e
\sim
\operatorname{Beta}(\alpha_y+t_{e,y},\beta_y+n_{e,y}-t_{e,y})
\]
with mean $\mu_{e,y}$ and variance $\sigma^2_{e,y}$. In the notation of
our paper, IFD uses the restricted approximation
\[
\Gamma_{\mathrm{IFD}}(x,y,C_e)
=
\mu_{e,y}.
\]
This explicit profile improves interpretability, data efficiency, and
invariance because the rejector receives only role-indexed summaries
such as competence at the model-top class or expert-best class. The
limitation is the missing $x$ dependence: all examples with the same
candidate class share the same competence estimate.
\paragraph{\RACER{} separates competence estimation from system-level
routing.}
\RACER{} keeps the one-stage L2D philosophy but changes the expert
interface. The competence channel estimates the posterior-predictive
role law
\[
\Gamma_\theta(x,y,C_e)
\approx
\Gamma^\star(x,y,C_e)
=
\Prob(M_E=y\mid X=x,Y=y,C_E=C_e)
\]
using a proper expert-correctness loss. The model posterior and
competence estimate are then combined as
\[
\widehat q_\omega(x,C_e)
=
\sum_y p_\omega(y\mid x)\Gamma_\theta(x,y,C_e).
\]
The final deferral decision compares $\widehat q_\omega(x,C_e)$ with
$p_{\omega,\max}(x)$, either directly through the plug-in threshold or
through an invariant decision head. This design preserves the benefit of
joint training. The classifier, competence estimator, and deferral head
can still be optimized together for system performance. But the
competence estimate is not identified only through the augmented action
coordinate. It has its own probability target. Detaching both the loss
weight and the deferral-head inputs blocks routing gradients to its dedicated
parameters, while shared representations still change through joint training:
\[
-\log \Pi_y(x,C_e)
-
\sg[\Gamma_\theta(x,y,C_e)]\log \Pi_\perp(x,C_e).
\]
This does not establish calibration of the trained system. The quantity
reported and compared as expert correctness is the probability-scale,
role-relative estimate
$\widehat q_\omega(x,C_e)$.
\paragraph{Why this matters for ID and OOD experts.}
For ID experts, query conditioning can reduce the approximation bias of
classwise IFD when competence varies within a class. Direct competence
supervision supplies an additional modelling constraint relative to a latent
routing interface; a universal variance reduction is not established here.
Under population shift, coherent invariance excludes absolute class-coordinate
shortcuts, but does not by itself guarantee better routing. L2D-Pop can learn a correct
ID margin through absolute class-coordinate or context-composition
shortcuts. IFD removes those shortcuts but cannot adapt within class.
\RACER{} combines the two desirable properties: it is role-aligned like
IFD and query-conditioned like L2D-Pop. Therefore, when an unseen expert
has the same structural competence pattern under different roles, or
when competence depends on query-local subtypes, \RACER{} has the
appropriate inductive bias:
\[
\text{learn competence relative to the candidate role, then compare it
to model correctness.}
\]
\paragraph{When the distinction disappears.}
In the infinite-data, realizable, perfectly optimized setting, L2D-Pop's
augmented softmax margin can recover the Bayes defer/predict comparison.
If one only cares about the binary decision and the test experts are
drawn from the same distribution, this may be enough. The distinction
becomes important when we ask for more: probability-scale
expert-correctness estimates, budget-swept rankings, multi-expert
comparability, and robust transfer to experts whose classwise or
instancewise competence differs from the training population.
\subsection{Identity-free and equivariant architectures}
The coherent-relabelling requirement in L2D is an instance of a broader
principle: the architecture should respect the symmetries of the task.
If class names are arbitrary, then renaming classes coherently should
not change a binary defer/predict decision. IFD enforces this by passing
only role-indexed and symmetric summaries to the rejector. \RACER{}
extends this idea to instance-dependent competence estimation. Candidate
labels enter only through role-relative relations such as $\1{y_i^C=y}$,
ranks, posterior values at roles, and correctness indicators. No learned
class embedding or untied class-specific channel is used.
This perspective is related to permutation-invariant and equivariant
neural architectures, including DeepSets and Set Transformers
\citep{zaheer2017deep,lee2019set}. However, our invariance is not merely
permutation invariance over context examples; it is coherent relabelling
invariance over class names. The context set may be reordered, and the
class labels may be renamed, yet the predicted expert correctness and
deferral decision should remain semantically unchanged.

\subsection{Real-human expert benchmarks}
A practical difficulty in evaluating L2D is the scarcity of datasets
with persistent expert annotations. Many common classification datasets
provide only a single label, aggregate votes, or sparse crowd
annotations. Clean deferral evaluation requires either a fixed expert
with labels across examples or enough repeated annotator behavior to
construct expert contexts. For this reason, the completed experiments in this manuscript focus on
radiologist and human--AI benchmarks with persistent annotations. These datasets test the deployment-relevant problem
of estimating expert competence from finite context under realistic
expert variability and annotation noise.

\section{Descriptive comparisons and variability}
\label{app:significance}
The real-benchmark summaries are descriptive. CheXpert standard deviations
range over 27 cells (three candidate humans, three held-out AI systems, and
three seeds); VinDr-CXR standard deviations in the main summary range over
five seeds. Cells sharing seeds, images, experts, or pathologies need not be
independent. We omit the earlier cell-level inferential analysis pending a
justified treatment of that dependence. No significance conclusion is drawn
from these summaries.

\section{Additional protocol notes}
\label{app:additional-protocol-notes}
\label{app:real-protocol-details}

The completed result bundle includes full group-level and per-pathology
breakdowns. The main text reports compact aggregate tables because the
fine-grained real-expert tables are large.

\subsection{Implementation and evaluation protocol}
\label{app:implementation-evaluation}

\paragraph{Training objectives and inference scores.}
Let $L_{\mathrm{aug}}(w)\coloneqq -\log\Pi_y-w\log\Pi_\perp$, and let
$L_{\mathrm{CE}}$ denote class cross-entropy. The main implementations use
$\lambda_{\mathrm{comp}}=1$; chest-radiography \RACERC{} additionally uses
$\lambda_{\mathrm{post}}=1$ and image-only CE weight $0.1$.
\begin{table}[H]
\centering\small
\begin{tabularx}{\linewidth}{@{}p{0.19\linewidth}XX@{}}
\toprule
Method & Training objective & Score swept at inference \\
\midrule
Classifier-conf. & $L_{\mathrm{CE}}$ & $-p_{\max}$ \\
IFD-score & $L_{\mathrm{CE}}$ & $\widehat q-p_{\max}$ \\
IFD-MLP & $L_{\mathrm{aug}}(\mathrm{LCB}_{y}\1{y=k_{\mathrm{best}}})$ & $d-\max_y f_y$ \\
L2D-Pop QI/QC & $L_{\mathrm{aug}}(\1{m=y})$ & $d-\max_y f_y$ \\
\RACERKNN{} & $L_{\mathrm{aug}}(\sg[\Gamma_y])$ & $\widehat q-p_{\max}$ \\
\RACERKernel{} / DeepSets & $L_{\mathrm{aug}}(\sg[\Gamma_y])+L_{\mathrm{comp}}$ & $\widehat q-p_{\max}$ \\
\RACERC{} (CXR) & Kernel objective with adapted logits, plus $L_{\mathrm{post}}+0.1L_{\mathrm{CE,base}}$ & $\widehat q-\max_y\pi_y$ \\
\bottomrule
\end{tabularx}
\caption{Implemented objectives and budget-ranking scores. The auxiliary RACER
routing head is trained but its logit is not the primary RACER budget score.
Multi-label losses are averaged over pathologies with valid targets.}
\label{tab:implementation-protocol}
\end{table}
The image encoder and classifier are updated during training; context image
features are evaluated without gradients, while query features and posterior
summaries remain differentiable for competence training. Dedicated RACER
competence parameters receive BCE gradients; routing gradients are stopped at
both the competence loss weight and the deferral-head inputs. The separate
context-posterior stress test also adds image-only auxiliary CE when its
configured weight is positive.

\paragraph{Context sources, separation, and expert selection.}
\begin{table}[H]
\centering\small
\begin{tabularx}{\linewidth}{@{}p{0.19\linewidth}XX@{}}
\toprule
Protocol & Context source and separation & Selection and evaluation \\
\midrule
Synthetic & Training contexts exclude the current minibatch. Evaluation
samples labelled context within the validation/test split and removes those
record indices from its query set. & Separate training for each $B$;
checkpoint selection uses unseen-ID validation AURSAC. Test queries are scored
using their images and expert contexts. \\
VinDr-CXR & Disjoint image splits; each expert's evaluation context and queries
are disjoint subsets of images they annotated. Up to 64 context images are used,
leaving at least one query. & Leave-one-reader-out agreement target; disagreement
between the other readers is masked. R8/R9/R10 define a co-annotation-cluster
holdout. \\
CheXpert & Disjoint record splits; training contexts contain 64 examples.
Validation/test contexts occupy half their respective splits, with disjoint
query remainders. & One candidate human plus AI experts; the held-out expert is
an AI system. Calibration uses the expert selected by each method. \\
\bottomrule
\end{tabularx}
\caption{Information and split protocol in the supplied implementations.
Contexts contain historical image labels and expert responses. Query outcomes
are used to score predictions, not as competence-head inputs at evaluation.}
\label{tab:data-protocol}
\end{table}
In CheXpert's multi-expert evaluation, IFD-score and RACER choose the available
expert with highest $\widehat q$; IFD-MLP and L2D-Pop choose by their learned
routing scores. Classifier confidence uses reproducible random expert selection
and ranks queries by classifier uncertainty. Expert availability is determined
by the annotation mask. Brier and ECE therefore compare different selected
expert--query pairs and jointly reflect expert selection and correctness
estimation, rather than a fixed-pair calibration comparison. Reported best
budgets and best accuracies summarize the test curves retrospectively; they
are not validation-selected operating points.

\subsection{VinDr-CXR protocol details}
\label{app:vindr-protocol-details}

The most important VinDr-CXR caveat is that the official test labels are
consensus-only and have no \texttt{rad\_id}. All L2D evaluation therefore uses
the multi-radiologist train annotations with image-level and expert-level splits.
For radiologist $e$, the expert prediction is $e$'s own binary annotation and
the target is the leave-one-radiologist-out consensus of the other two readers.
If those two readers disagree, the pathology is masked. The all-rater majority
label is used only for classifier-only training and image-level stratification.

VinDr-CXR does not provide a hospital identifier in the annotation file, so the
co-annotation-cluster holdout is defined from co-annotation frequencies. The dominant
co-annotation cluster is R8/R9/R10; we use this as an inferred OOD expert group.
This is a radiologist/co-annotation-cluster shift proxy, not a verified hospital
label.

\begin{table}[H]
\centering
\small
\setlength{\tabcolsep}{5pt}
\renewcommand{\arraystretch}{1.05}
\begin{tabular}{lcccc}
\toprule
Group & Classifier acc. & Expert acc. & Expert bal. acc. & Classifier AURSBAC \\
\midrule
Overall & $0.9026\pm0.0004$ & $0.9700\pm0.0003$ & $0.8365\pm0.0032$ & $0.6878\pm0.0078$ \\
Seen & $0.9860\pm0.0015$ & $0.9942\pm0.0010$ & $0.8314\pm0.0108$ & $0.7057\pm0.0152$ \\
Unseen-ID & $0.9628\pm0.0060$ & $0.9862\pm0.0041$ & $0.8323\pm0.0739$ & $0.7083\pm0.0466$ \\
Unseen-OOD & $0.8780\pm0.0006$ & $0.9628\pm0.0004$ & $0.8414\pm0.0014$ & $0.6898\pm0.0077$ \\
\bottomrule
\end{tabular}
\caption{VinDr-CXR expert-strength diagnostics for the classifier-confidence
baseline. Values are mean $\pm$ standard deviation over the five final seeds.
The deferred radiologist is substantially more accurate than the image-only
classifier, especially on unseen-OOD cases, which explains why a simple
classifier-confidence deferral rule is competitive. Expert balanced accuracy is
reported because raw multi-label accuracy is inflated by frequent negative
pathology labels.}
\label{tab:vindr_expert_accuracy}
\end{table}

\begin{table}[H]
\centering
\small
\setlength{\tabcolsep}{5pt}
\renewcommand{\arraystretch}{1.05}
\begin{tabular}{llr@{\qquad}llr}
\toprule
\multicolumn{3}{c}{Top triples} & \multicolumn{3}{c}{Top pairs} \\
\cmidrule(lr){1-3}\cmidrule(lr){4-6}
Experts & & Count & Experts & & Count \\
\midrule
R8/R9/R10 && 5501 & R8/R10 && 5931 \\
R2/R3/R5 && 115 & R8/R9 && 5621 \\
R1/R2/R3 && 104 & R9/R10 && 5529 \\
R2/R3/R6 && 100 & R2/R3 && 1042 \\
R12/R13/R16 && 99 & R1/R2 && 911 \\
\bottomrule
\end{tabular}
\caption{VinDr-CXR co-annotation structure used to define the inferred
R8/R9/R10 OOD expert cluster.}
\label{tab:vindr_coannotation}
\end{table}

\subsection{CheXpert multi-expert protocol details}
\label{app:chexpert-protocol-details}

CheXpert is evaluated as a multi-label binary-relevance task. Each run treats one
human reader as the candidate human expert, includes two of the three external AI
systems as seen experts during training and validation, and holds out the
remaining AI system until test time. The primary candidate-reader set is
$\{\texttt{bc1},\texttt{bc5},\texttt{bc7}\}$; \texttt{bc2} and \texttt{bc3}
are retained for consensus construction. For a deferred human reader, the target
is formed from the remaining human readers, so the deferred reader's own
annotation is not used as its target. The same reference target is used for
every candidate AI system in that run: a positive label requires at least
three positive votes among the four remaining human readers, with all four
votes available. The full grid has three candidate human
experts, three held-out AI choices, and three seeds, for $27$ matched cells per
method. All adaptive methods train with $B=64$; the multi-expert evaluator
uses half of each validation/test image split as context (67/66 images for the
saved 134/133-image splits), with the remainder as queries, rather than
enforcing the training context size.

The three external AI systems are CARZero, KAD, and CheXZero. They were chosen
because they are pretrained chest-radiograph systems, rather than small models
trained only inside our experiment, and they represent strong contemporary CXR
recognition paradigms: zero-shot or language-aligned CXR prediction and
knowledge-guided chest-radiograph diagnosis. This makes the AI deferral targets
plausible high-performing clinical-imaging experts. In each run, two of these
AI systems are available during training and validation, while the third is held
out until test time; rotating the held-out system tests whether a router trained
with human and AI context can generalize to an unseen pretrained CXR model.

\subsection{Extended CheXpert metrics}
\label{app:real-extended-tables}

\begin{table*}[H]
\centering
\scriptsize
\setlength{\tabcolsep}{3.5pt}
\renewcommand{\arraystretch}{1.06}
\resizebox{\textwidth}{!}{%
\begin{tabular}{lccccccccc}
\toprule
Method & Seen+unseen & Seen & AURSAC & Best acc. & Best budget & Model acc. & Expert acc. & Brier & ECE \\
\midrule
\RACERC{} & \best{$0.7243\pm0.0341$} & $0.7244\pm0.0328$ & $0.9008\pm0.0068$ & $0.9128\pm0.0082$ & $0.1259\pm0.1305$ & \best{$0.9042\pm0.0127$} & $0.8896\pm0.0084$ & $0.0947\pm0.0071$ & $0.0418\pm0.0134$ \\
\RACERKernel{} & $0.7241\pm0.0259$ & \best{$0.7276\pm0.0224$} & $0.9030\pm0.0075$ & \best{$0.9149\pm0.0080$} & $0.2041\pm0.2126$ & $0.9029\pm0.0090$ & $0.8907\pm0.0092$ & $0.0916\pm0.0068$ & $0.0366\pm0.0098$ \\
IFD-MLP \citep{strong2026identity} & $0.7105\pm0.0405$ & $0.7139\pm0.0349$ & $0.9012\pm0.0093$ & $0.9083\pm0.0095$ & $0.3133\pm0.3320$ & $0.9011\pm0.0109$ & $0.8947\pm0.0083$ & $0.0861\pm0.0061$ & \best{$0.0247\pm0.0075$} \\
IFD-score \citep{strong2026identity} & $0.7013\pm0.0444$ & $0.7092\pm0.0431$ & $0.9057\pm0.0060$ & $0.9111\pm0.0062$ & $0.2930\pm0.2901$ & $0.8967\pm0.0102$ & \best{$0.8993\pm0.0065$} & $0.0849\pm0.0052$ & $0.0293\pm0.0067$ \\
L2D-QC \citep{tailor2024population} & $0.7007\pm0.0219$ & $0.7026\pm0.0202$ & $0.9014\pm0.0089$ & $0.9119\pm0.0100$ & $0.2200\pm0.2297$ & $0.9030\pm0.0107$ & $0.8861\pm0.0080$ & -- & -- \\
L2D-QI \citep{tailor2024population} & $0.6986\pm0.0265$ & $0.7029\pm0.0224$ & $0.9035\pm0.0098$ & $0.9118\pm0.0098$ & $0.2174\pm0.1581$ & $0.9012\pm0.0100$ & $0.8905\pm0.0095$ & -- & -- \\
\RACERKNN{} & $0.6956\pm0.0360$ & $0.6987\pm0.0348$ & \best{$0.9076\pm0.0062$} & $0.9136\pm0.0059$ & $0.2985\pm0.2787$ & $0.9002\pm0.0088$ & $0.8991\pm0.0074$ & \best{$0.0844\pm0.0061$} & $0.0534\pm0.0088$ \\
Classifier-conf. & $0.6931\pm0.0360$ & $0.7001\pm0.0340$ & $0.8836\pm0.0058$ & $0.9099\pm0.0097$ & $0.0826\pm0.0396$ & $0.8999\pm0.0102$ & $0.8538\pm0.0067$ & -- & -- \\
\bottomrule
\end{tabular}%
}
\caption{Extended CheXpert multi-expert metrics corresponding to
\Cref{tab:chexpert-multiexpert-results}. The main text reports the primary
balanced-accuracy and calibration columns; secondary accuracy and budget
summaries are included here for completeness.}
\label{tab:chexpert-multiexpert-extended}
\end{table*}

\section{Ablations}

\subsection{Ablation: Role-relative DeepSets competence estimator}

\label{alb:deepsets}

To isolate the contribution of explicit same-role kernel pooling, we implemented
RACER-DeepSets, a permutation-invariant neural competence estimator trained with the
same BCE competence target as RACER-kernel. RACER-DeepSets receives only
role-admissible query--context features: candidate-role equality indicators, context
correctness, query--context similarity, posterior values and ranks, and symmetric
context summaries. It uses no class embeddings, class-specific heads, or absolute
class-coordinate channels, preserving the coherent-relabelling constraint used by
Racer.

Table~\ref{tab:pathmnist_deepsets_gain} reports AURSAC gain over classifier confidence
for all methods and context sizes. RACER-DeepSets is the strongest non-kernel neural
ablation, but RACER-kernel is consistently better at moderate and large context sizes.
In particular, at \(B=1000\), RACER-kernel reaches \(+0.0273\) overall AURSAC gain,
compared with \(+0.0194\) for RACER-DeepSets.

Table~\ref{tab:pathmnist_deepsets_calibration} reports calibration metrics. RACER-kernel
also gives lower Brier score at all nontrivial context sizes and remains substantially
better at large context sizes, reaching \(0.2029\) at \(B=1000\), compared with \(0.2197\)
for RACER-DeepSets. These results suggest that the main conceptual contribution is the
role-relative probability-scale competence interface, while explicit same-role kernel
pooling provides a useful finite-context smoothing bias.

\begin{table*}[h]
\centering
\small
\setlength{\tabcolsep}{3.5pt}
\renewcommand{\arraystretch}{1.04}
\resizebox{\textwidth}{!}{%
\begin{tabular}{lccccccc}
\toprule
Method & $B=9$ & $B=25$ & $B=50$ & $B=100$ & $B=200$ & $B=500$ & $B=1000$ \\
\midrule
IFD-score & $-0.0084\pm0.0023$ & $-0.0119\pm0.0072$ & $+0.0000\pm0.0081$ & $-0.0055\pm0.0078$ & $-0.0012\pm0.0047$ & $+0.0025\pm0.0045$ & $+0.0044\pm0.0052$ \\
IFD-MLP & $-0.0143\pm0.0056$ & $-0.0113\pm0.0036$ & $-0.0045\pm0.0037$ & $-0.0068\pm0.0020$ & $-0.0043\pm0.0096$ & $-0.0020\pm0.0049$ & $+0.0003\pm0.0062$ \\
L2D-Pop (QI) & $-0.0099\pm0.0045$ & $-0.0191\pm0.0116$ & $-0.0100\pm0.0023$ & $-0.0192\pm0.0162$ & $-0.0112\pm0.0117$ & $-0.0049\pm0.0042$ & $-0.0061\pm0.0040$ \\
L2D-Pop (QC) & $-0.0060\pm0.0013$ & $-0.0146\pm0.0088$ & $-0.0076\pm0.0056$ & $-0.0061\pm0.0023$ & $-0.0050\pm0.0013$ & $-0.0042\pm0.0038$ & $-0.0092\pm0.0070$ \\
\RACERKNN{} & $-0.0086\pm0.0034$ & $-0.0128\pm0.0081$ & $-0.0049\pm0.0083$ & $-0.0039\pm0.0065$ & $-0.0008\pm0.0102$ & $+0.0058\pm0.0124$ & $+0.0147\pm0.0064$ \\
\RACERKernel{} & $-0.0080\pm0.0034$ & $\mathbf{-0.0006}\pm0.0036$ & $\mathbf{+0.0098}\pm0.0144$ & $\mathbf{+0.0129}\pm0.0025$ & $\mathbf{+0.0186}\pm0.0037$ & $\mathbf{+0.0251}\pm0.0068$ & $\mathbf{+0.0273}\pm0.0051$ \\
RACER-DeepSets & $\mathbf{-0.0057}\pm0.0058$ & $-0.0041\pm0.0072$ & $+0.0074\pm0.0082$ & $+0.0019\pm0.0091$ & $+0.0156\pm0.0085$ & $+0.0204\pm0.0017$ & $+0.0194\pm0.0029$ \\
\bottomrule
\end{tabular}%
}
\caption{PathMNIST context-scaling AURSAC gain over the classifier-confidence router. Entries are mean $\pm$ standard deviation over three seeds, paired by seed against classifier confidence at the same context size. The classifier baseline is omitted because its gain is zero by definition. Bold denotes the best non-classifier method at each context size.}
\label{tab:pathmnist_deepsets_gain}
\end{table*}

\begin{table*}[h]
\centering
\footnotesize
\setlength{\tabcolsep}{3pt}
\renewcommand{\arraystretch}{1.06}
\resizebox{\textwidth}{!}{%
\begin{tabular}{llccccccc}
\toprule
Method & Metric & $B=9$ & $B=25$ & $B=50$ & $B=100$ & $B=200$ & $B=500$ & $B=1000$ \\
\midrule
\multirow{2}{*}{IFD-score} & Brier $\downarrow$ & $0.2560\pm0.0016$ & $0.2519\pm0.0013$ & $0.2416\pm0.0004$ & $0.2339\pm0.0010$ & $0.2284\pm0.0004$ & $0.2237\pm0.0007$ & $0.2224\pm0.0006$ \\
 & ECE $\downarrow$ & $0.1054\pm0.0046$ & $0.1358\pm0.0027$ & $0.1064\pm0.0003$ & $0.0807\pm0.0059$ & $0.0567\pm0.0015$ & $0.0328\pm0.0003$ & $0.0225\pm0.0018$ \\
\addlinespace[1pt]
\multirow{2}{*}{IFD-MLP} & Brier $\downarrow$ & $0.2563\pm0.0016$ & $0.2516\pm0.0019$ & $0.2417\pm0.0007$ & $0.2339\pm0.0012$ & $0.2283\pm0.0005$ & $0.2236\pm0.0008$ & $0.2222\pm0.0004$ \\
 & ECE $\downarrow$ & $0.1055\pm0.0047$ & $0.1351\pm0.0040$ & $0.1067\pm0.0009$ & $0.0808\pm0.0054$ & $0.0567\pm0.0029$ & $0.0320\pm0.0007$ & $\mathbf{0.0214}\pm0.0026$ \\
\addlinespace[1pt]
\multirow{2}{*}{\RACERKNN{}} & Brier $\downarrow$ & $0.4302\pm0.0062$ & $0.3013\pm0.0035$ & $0.2572\pm0.0002$ & $0.2387\pm0.0009$ & $0.2306\pm0.0005$ & $0.2253\pm0.0019$ & $0.2221\pm0.0013$ \\
 & ECE $\downarrow$ & $0.4252\pm0.0068$ & $0.2171\pm0.0059$ & $0.1441\pm0.0023$ & $0.0978\pm0.0051$ & $0.0732\pm0.0015$ & $0.0633\pm0.0053$ & $0.0572\pm0.0025$ \\
\addlinespace[1pt]
\multirow{2}{*}{\RACERKernel{}} & Brier $\downarrow$ & $0.2251\pm0.0007$ & $\mathbf{0.2233}\pm0.0013$ & $\mathbf{0.2200}\pm0.0010$ & $\mathbf{0.2168}\pm0.0006$ & $\mathbf{0.2102}\pm0.0018$ & $\mathbf{0.2054}\pm0.0010$ & $\mathbf{0.2029}\pm0.0006$ \\
 & ECE $\downarrow$ & $0.0122\pm0.0094$ & $\mathbf{0.0131}\pm0.0035$ & $\mathbf{0.0220}\pm0.0076$ & $\mathbf{0.0365}\pm0.0027$ & $\mathbf{0.0302}\pm0.0047$ & $\mathbf{0.0249}\pm0.0061$ & $0.0255\pm0.0070$ \\
\addlinespace[1pt]
\multirow{2}{*}{RACER-DeepSets} & Brier $\downarrow$ & $\mathbf{0.2249}\pm0.0007$ & $0.2250\pm0.0012$ & $0.2219\pm0.0017$ & $0.2191\pm0.0032$ & $0.2188\pm0.0031$ & $0.2209\pm0.0047$ & $0.2197\pm0.0022$ \\
 & ECE $\downarrow$ & $\mathbf{0.0107}\pm0.0010$ & $0.0144\pm0.0146$ & $0.0363\pm0.0098$ & $0.0398\pm0.0226$ & $0.0445\pm0.0162$ & $0.0534\pm0.0320$ & $0.0466\pm0.0176$ \\
\bottomrule
\end{tabular}%
}
\caption{PathMNIST context-scaling expert-correctness calibration. Entries are mean $\pm$ standard deviation over three seeds. Brier score and 15-bin ECE are reported only for methods that output probability-scale expert-correctness estimates. Bold denotes the best value at each context size for the corresponding metric.}
\label{tab:pathmnist_deepsets_calibration}
\end{table*}

\end{document}